\documentclass{article}

\usepackage{microtype}
\usepackage{graphicx}
\usepackage{subfigure}
\usepackage{booktabs} % for professional tables

\usepackage{hyperref}

\usepackage[accepted]{icml2024}

\usepackage{amsmath}
\usepackage{amssymb}
\usepackage{mathtools}
\usepackage{amsthm}

\usepackage[capitalize,noabbrev]{cleveref}

\usepackage{algorithm}
\usepackage{algorithmic}

\usepackage{amsfonts}       % blackboard math symbols
\usepackage{nicefrac}       % compact symbols for 1/2, etc.
\usepackage{microtype}      % microtypography
\usepackage{xcolor}         % colors
\usepackage{amsmath}
\usepackage{amssymb}
\usepackage{multicol}
\usepackage{graphicx,lipsum}
\usepackage{amsthm}
\usepackage{stackengine}
\usepackage{algorithm}
\usepackage{algorithmic}

\usepackage{multirow}
\usepackage{bbm}

\newcommand{\norm}[1]{\left\lVert#1\right\rVert}

\DeclareMathOperator*{\argmax}{argmax}

\DeclareMathOperator*{\softmax}{softmax}

\theoremstyle{plain}
\newtheorem{theorem}{Theorem}[section]
\newtheorem{proposition}[theorem]{Proposition}

\theoremstyle{definition}
\newtheorem{definition}[theorem]{Definition}

\theoremstyle{remark}

\usepackage[textsize=tiny]{todonotes}

\icmltitlerunning{Understanding and Diagnosing Deep Reinforcement Learning}

\begin{document}

\twocolumn[
\icmltitle{Understanding and Diagnosing Deep Reinforcement Learning}

\icmlsetsymbol{equal}{*}

\begin{icmlauthorlist}
\icmlauthor{Ezgi Korkmaz}{a}
\end{icmlauthorlist}

\icmlaffiliation{a}{University College London (UCL)}

\icmlcorrespondingauthor{Ezgi Korkmaz}{ezgikorkmazmail@gmail.com}

\icmlkeywords{Machine Learning, ICML}

\vskip 0.3in
]



\printAffiliationsAndNotice{}  % leave blank if no need to mention equal contribution

\begin{abstract}
\vskip-0.01in
Deep neural policies have recently been installed in a diverse range of settings, from biotechnology to automated financial systems. However, the utilization of deep neural networks to approximate the value function leads to concerns on the decision boundary stability, in particular, with regard to the sensitivity of policy decision making to indiscernible, non-robust features due to highly non-convex and complex deep neural manifolds. These concerns constitute an obstruction to understanding the reasoning made by deep neural policies, and their foundational limitations.
 Hence, it is crucial to develop techniques that aim to understand the sensitivities in the learnt representations of neural network policies.
 To achieve this we introduce a theoretically founded method that provides a systematic analysis of the unstable directions in the deep neural policy decision boundary across both time and space.
 Through experiments in the Arcade Learning Environment (ALE), we demonstrate the effectiveness of our technique for identifying correlated directions of instability, and for measuring how sample shifts remold the set of sensitive directions in the neural policy landscape.
 Most importantly, we demonstrate that state-of-the-art robust training techniques yield learning of disjoint unstable directions, with dramatically larger oscillations over time, when compared to standard training.
 We believe our results reveal the fundamental properties of the decision process made by reinforcement learning policies, and can help in constructing reliable and robust deep neural policies.
\vskip-0.09in
\end{abstract}
\section{Introduction}
Reinforcement learning algorithms leveraging the power of deep neural networks have obtained state-of-the-art results initially in game-playing tasks \citep{mn15} and subsequently in continuous control \citep{lil15}. Since this initial success, there has been a continuous stream of developments both of new algorithms \citep{mnih16, hado16, wang16}, and striking new performance records in highly complex tasks \citep{silver17,jul20nat}. While the field of deep reinforcement learning has developed rapidly \citep{daniel23}, the understanding of the representations learned by deep neural network policies has lagged behind.

The lack of understanding of deep neural policies is of critical importance in the context of the sensitivities of policy decisions to imperceptible, non-robust features. Beginning with the work of \citep{szegedy13, fellow15}, deep neural networks have been shown to be vulnerable to adversarial perturbations below the level of human perception. In response, a line of work has focused on proposing training techniques to increase robustness by applying these perturbations to the input of deep neural networks during training time (i.e. adversarial training) \citep{fellow15, madry17}.
Yet, concerns have been raised on these methods including decreased accuracy on clean data \citep{arjun19}, prohibiting generalization \citep{korkmaz2023aaai}, and incorrect invariance to semantically meaningful changes \citep{tramer20nips}. While some studies argued that detecting adversarial directions could be the best we can do so far \citep{korkmaz2023icml}, the diagnostic perspective on understanding policy decision making and vulnerabilities requires urgent further attention.

Thus, it is crucial to develop techniques to precisely understand and diagnose the sensitivities of deep neural policies, in order to effectively evaluate newly proposed algorithms and training methods. In particular, there is a need to have diagnostic methods that can automatically identify policy sensitivities and instabilities that arise under many different scenarios, without requiring extensive research effort for each new instance.

For this reason, in our paper we focus on understanding the learned representations and policy vulnerabilities and ask the following questions:
\emph{(i) How can we analyze the rationale behind deep reinforcement learning decisions?}
\emph{(ii) What is the temporal and spatial relation between non-robust directions on the deep neural policy manifold?}
\emph{(iii) How do the directions of instabilities in the deep neural policy landscape transform under a portfolio of state-of-the-art adversarial attacks?}
\emph{(iv) How does distributional shift affect the learnt non-robust representations in reinforcement learning with high dimensional state representation MDPs?}
\emph{(v) Does the state-of-the-art certified adversarial training solve the problem of learning correlated non-robust representations in sequential decision making?}
To be able to answer these questions in our paper from worst-case to natural directions we focus on understanding the representations learned by deep reinforcement learning policies and make the following contributions:

\begin{itemize}
\item We introduce a theoretically founded novel approach to systematically discover and analyze the spatial and temporal correlation of directions of instability on the deep reinforcement learning manifold.
\item We highlight the connection between neural processing with visual illusion stimulus and our analysis to understand and diagnose deep neural policies.
We conduct extensive experiments in the Arcade Learning Environment with neural policies trained in high-dimensional state representations, and provide an analysis on a portfolio of state-of-the-art adversarial attack techniques. Our results demonstrate the precise effects of adversarial attacks on the non-robust features learned by the policy.
\item We investigate the effects of distributional shift on the correlated vulnerable representation patterns learned by deep reinforcement learning policies to provide a comprehensive and systematic robustness analysis of deep neural policies.
\item Finally, our results demonstrate the presence of non-robust features in adversarially trained deep reinforcement learning policies, and that the state-of-the-art certified robust training methods lead to learning disjoint and spikier vulnerable representations.
\end{itemize}

\section{Background and Preliminaries}

\subsection{Preliminaries}
A Markov Decision Process (MDP) is defined by a tuple $(\mathcal{S}, \mathcal{A}, \mathcal{P}, \mathcal{R}, \gamma)$ where $\mathcal{S}$ is a set of states, $\mathcal{A}$ is a set of actions, $\mathcal{P}:\mathcal{S}\times\mathcal{A}\times\mathcal{S} \to [0,1]$ is the Markov transition kernel,
$\mathcal{R}:\mathcal{S}\times\mathcal{A}\times\mathcal{S} \to \mathbb{R}$ is the reward function, and $\gamma \in [0,1)$ is the discount factor.
A reinforcement learning agent interacts with an MDP by observing the current state $s\in \mathcal{S}$ and taking an action $a \in \mathcal{A}$. The agent then transitions to state $s'$ with probability $\mathcal{P}(s,a,s')$ and receives reward $\mathcal{R}(s,a,s')$. A policy $\pi:\mathcal{S}\times\mathcal{A} \to [0,1]$ selects action $a$ in state $s$ with probability $\pi(s,a)$. The main objective in reinforcement learning is to learn a policy $\pi$ which maximizes the expected cumulative discounted rewards
$R = \mathbb{E}_{a_t \sim \pi(s_t,\cdot)} \sum_t \gamma^t \mathcal{R}(s_t,a_t,s_{t+1})$.
This maximization is achieved by iterative Bellman update to learn a state-action value function \cite{watkins92}
\[
 Q(s_t,a_t) = \mathcal{R}(s_t,a_t,s_{t+1}) + \gamma \sum_{s_t} \mathcal{P}(s_{t+1}|s_t,a_t) V(s_{t+1}).
\]
$Q(s,a)$ converges to the optimal state-action value function, representing the expected cumulative discounted rewards obtained by the optimal policy when starting in state $s$ and taking action $a$ with value function $V(s) = \max_{a \in \mathcal{A}}Q(s,a)$.
Hence, the optimal policy $\pi^*(s,a)$ can be obtained by executing the action $a^*(s) = \argmax_a Q(s,a)$, i.e. the action maximizing the state-action value function in state $s$.

\subsection{Adversarial Perturbation Techniques and Formulations}
\label{advpert}
Following the initial study conducted by \citet{szegedy13}, \citet{fellow15} proposed a fast and efficient way to produce $\epsilon$-bounded adversarial perturbations in image classification based on linearization of $J(x,y)$, the cost function used to train the network, at data point $x$ with label $y$. Consequently, \citet{kurakin16} proposed the iterative form of this algorithm: the iterative fast gradient sign method (I-FGSM).
\begin{equation}
x_{\textrm{adv}}^{N+1} = \textrm{clip}_\epsilon(x_{\textrm{adv}}^N +\alpha \textrm{sign}(\nabla_x J(x^N_{\textrm{adv}},y)))
\end{equation}
This algorithm further has been improved by the proposal of the utilization of the momentum term \citep{don18}. Following this \citet{korkmaz20} proposed a Nesterov momentum technique to compute $\epsilon$-bounded adversarial perturbations for deep reinforcement learning policies by computing the gradient at the point $s_{\textrm{adv}}^t +  \mu \cdot v_t$,
\begin{align}
v_{t+1} = \mu \cdot v_t &+ \dfrac{\nabla_{s_\textrm{adv}}J(s_{\textrm{adv}}^t +  \mu \cdot v_t ,a)}{\lVert\nabla_{s_\textrm{adv}}J(s_{\textrm{adv}}^t +  \mu \cdot v_t ,a)\rVert_1} \\
s_\textrm{adv}^{t+1} &= s_\textrm{adv}^{t} + \alpha \cdot  \dfrac{v_{t+1}}{\lVert v_{t+1}\rVert_2}
\end{align}
Another class of algorithms for computing adversarial perturbations focuses on different methods for computing the smallest possible perturbation which successfully changes the output of the target function. The DeepFool method of \citet{dezfool16} works by repeatedly computing projections to the closest separating hyperplane of a linearization of the deep neural network at the current point. \citet{carlini17} proposed targeted adversarial formulations in image classification based on distance minimization between the original sample and the adversarial sample
\begin{equation}
\mathnormal{\min_{x_{\textrm{adv}} \in \mathcal{X}} c\cdot J(x_{\textrm{adv}}) + \norm{x_{\textrm{adv}}-x}_2^2}
\label{carlini}
\end{equation}
Another variant of this algorithm is based on $\ell_1$-regularization of the $\ell_2$-norm bounded \citet{carlini17} adversarial formulation \cite{ead18}.
\[
\mathnormal{\min_{x_{\textrm{adv}} \in \mathcal{X}} c\cdot  J(x_{\textrm{adv}}) + \sigma_1\norm{x_{\textrm{adv}}-x}_1 + \sigma_2\norm{x_{\textrm{adv}}-x}_2^2}
\]

\subsection{Deep Reinforcement Learning Policies and Adversarial Effects}
\label{advrl}
Beginning with the work of \citet{huang17} and \citet{kos17}, which introduced adversarial examples based on FGSM to deep reinforcement learning, there has been a long line of research on both adversarial attacks and robustness for deep neural policies. 
On the attack side, \citet{korkmazuai} showed that it is possible to compute adversarial perturbations for robust deep reinforcement learning policies, and further proposed tools to interpret the non-robustness of deep neural policies.
More intriguingly, later study discovered that deep reinforcement learning policies learn similar adversarial directions across MDPs intrinsic to the training environment, thus revealing an underlying approximately linear structure learnt by deep neural policies \citep{korkmazaaai}.  
On the defense side \citet{pinto17} model the interaction between an adversary producing perturbations and the deep neural policy taking actions as a zero-sum game, and train the policy jointly with the adversary in order to improve robustness.
More recently, \citet{huan20} formalized the adversarial problem in deep reinforcement learning by introducing a modified MDP definition which they term State-Adversarial MDP (SA-MDP). Based on this model the authors proposed a theoretically motivated \textit{certified robust} adversarial training algorithm called SA-DQN.
Quite recently, \citet{korkmaz2023aaai} provided a contrast between natural directions and adversarial directions with respect to their perceptual similarity to base states and impact on the policy performance. While the results in this paper demonstrate that certified adversarial training techniques limit the generalization capabilities of deep reinforcement learning policies, the paper further argues the need for rethinking robustness in deep reinforcement learning.
While recent studies raised some concerns on the drawbacks of certified adversarial training techniques from generalization to security, these studies lack a method of explaining and understanding the main problems of robustness in deep reinforcement learning, and in particular with clear analysis of the vulnerabilities of the policies.

\section{Probing the Deep Neural Policy Manifold via Non-Lipschitz Directions}

In our paper our goal is to seek answers for the following questions:
\begin{itemize}
\item \textit{What is the reasoning behind deep reinforcement learning decision making?}
\item \textit{How can we analyze the robustness of deep reinforcement learning policies across time and space?}
\item \textit{What are the effects of distributional shift on the vulnerable representations learnt?}
\item \textit{How do adversarial attacks remold the volatile patterns learnt by the neural policies?}
\item \textit{Does adversarial training ensure learning robust and safe policies without any vulnerability?}
\end{itemize}
To be able to answer these questions we propose a principled robustness appraisal method that probes the deep reinforcement learning manifold via non-Lipschitz directions across time and across space. In the remainder of this section we explain in detail our proposed method.

\begin{definition}[\emph{$\epsilon$-non-Lipschitz Direction}]
  Let $Q$ be a state-action value function and let $\epsilon > 0$. For a state $s \in \mathcal{S}$ and vector $w \in\mathbb{R}^d$, let $\hat{s} = s+\epsilon w$. The vector $v$ is an $\epsilon$-non-Lipschitzness direction that uncovers the high-sensitivities of the deep neural manifold for $Q$ in state $s$ if
\vskip-0.2in
\begin{align}
\label{def1}
v = \argmax_{\lVert w\rVert_2=1} Q(\hat{s},  \argmax_{a \in \mathcal{A}} &Q(\hat{s},a)) \\
& - Q(\hat{s}, \argmax_{a \in \mathcal{A}} Q(s,a)). \nonumber
\end{align}
\end{definition}
\vskip-0.2in
In words, $v$ is a non-Lipschitz direction when adding a perturbation of $\ell_2$-norm $\epsilon$ along $v$ maximizes the difference between the maximum state-action value in the new state and the value assigned in the new state to the previously maximal action. Eqn \ref{def1} can be approximated by using the softmax cross entropy loss.\footnote{$\pi(s,a)$ is defined as the softmax policy of the state-action value function $\pi(s,a) = \dfrac{e^{(Q(s,a)/T)}}{\sum_{a' \in \mathcal{A}} e^{(Q(s,a')/T)}}$.}
The cross entropy loss between the softmax policy in state $s_g$ and the argmax policy $\tau(s,a) = \mathbbm{1}_{a = \argmax_{a'} \pi(s,a')}(a)$ at state $s$ is
\begin{align*}
J(s,s_g) &= -\sum_{a \in \mathcal{A}} \tau(s,a) \log(\pi(s_g,a))  \\
&= -\log(\pi(s_g,a^*(s))).
\end{align*}
Therefore by definition of the softmax policy we have
\begin{align*}
J(s,s_g) &= \log \sum_{a' \in \mathcal{A}} e^{Q(s_g,a')/T} - Q(s_g,a^*(s))/T  \\
&\approx (Q(s_g,a^*(s_g)) - Q(s_g,a^*(s)))/T
\end{align*}
where the final approximate equality becomes close to an equality as $T$ gets smaller. 
Setting $v = s_g - s$, shows that maximizing the softmax cross entropy approximates the maximization in Eqn \ref{def1}.
Hence, the gradient $\nabla_{s_g} J(s,s_g)\rvert_{s_g = s}$ gives the direction of the largest increase in cross-entropy when moving from state $s$. Intuitively this is the direction along which the policy distribution $\pi(s,a)$ will most rapidly diverge from the argmax policy. Hence, $\nabla_{s_g} J(s,s_g)\rvert_{s_g = s}$ is a high-sensitivity direction in the neural policy landscape in state $s$.
Fundamentally, moving along the non-Lipschitz directions on the deep neural policy decision boundary will uncover the non-robust features learnt by the reinforcement learning policy.
To capture the correlated non-robust features we must aggregate the information on high-sensitivity directions from a collection of states visited while utilizing the policy $\pi$ in a given MDP. We thus define a single direction which captures the aggregate non-robust feature information from multiple states via the first principal component of the non-Lipschitz directions as follows:
\begin{definition}[\emph{Principal non-Lipschitz direction}]
  \label{def:correlatedfeature}
  Given a set of $n$ states $S = \{s_i\}_{i=1}^{n}$ the principal non-Lipschitz direction is the vector $\mathcal{G}_S$ given by
  \[
    \mathcal{G}_S = \argmax_{\left\{z\in \mathbb{R}^d \mid \lVert z \rVert_2 =1\right\}} \frac{1}{n}\sum_{i=1}^{n} \langle z, \nabla_{s_g} J(s_i,s_g) \rvert_{s_g = s_i}\rangle^2.
  \]
\end{definition}
\begin{proposition}[\emph{Spectral characterization of principal non-Lipschitz directions}]
  \label{prop:eigenvector}
  Given a set of $n$ states $S = \{s_i\}_{i=1}^{n}$ define the matrix $\mathcal{L}(S)$ by
  \begin{equation*}
  \mathcal{L}(S) = \frac{1}{n}\sum_{i=1}^{n}  \nabla_{s_g} J(s_i,s_g)\rvert_{s_g = s_i}[\nabla_{s_g} J(s_i,s_g)\rvert_{s_g = s_i}]^{\top}.
\end{equation*}
  Then $\mathcal{G}_S$ is the eigenvector corresponding to the largest eigenvalue of $\mathcal{L}(S)$.
\end{proposition}
\begin{proof}
Observe that by linearity of the inner product
\begin{align*}
    & \frac{1}{n} \sum_{i=1}^{n}  \langle z, \nabla_{s_g} J(s_i,s_g)\rvert_{s_g = s_i}\rangle^2 \\
    &=  \frac{1}{n}\sum_{i=1}^{n} z^{\top} \nabla_{s_g} J(s_i,s_g)\rvert_{s_g = s_i}[\nabla_{s_g} J(s_i,s_g)\rvert_{s_g = s_i}]^{\top}z\\
    &=  z^{\top}(\frac{1}{n}\sum_{i=1}^{n}  \nabla_{s_g} J(s_i,s_g)\rvert_{s_g = s_i} [\nabla_{s_g} J(s_i,s_g)\rvert_{s_g = s_i}]^{\top})z  \\
    &= z^{\top}\mathcal{L}(S)z.
\end{align*}
Thus
$\mathcal{G}_S = \argmax_{\left\{z \in \mathbb{R}^d \mid \lVert z \rVert_2 =1\right\}} z^{\top}\mathcal{L}(S) z$.
Therefore, by the variational characterization of eigenvalues, $\mathcal{G}_S$ is the eigenvector corresponding to the largest eigenvalue of $\mathcal{L}(S)$.
\end{proof}

Thus, the dominant eigenvector corresponds to $\mathcal{G}_S$, the largest correlation with non-Lipschitz directions across time, which follows from the standard analysis of principal component analysis.
 Also note that $\mathcal{G}_S$ has the same dimensions as each state $s$, and thus can easily be rendered in the same format as the states to visualize non-robust features.
Proposition \ref{prop:eigenvector} shows that $\mathcal{G}_S$ can be computed by solving an eigenvalue problem. Proposition \ref{prop:eigenvector} is the basis for Algorithm \ref{hmap}, which computes $\mathcal{G}_S$ by first calculating $\mathcal{L}(S)$ by summing over states, and then outputs the maximum eigenvector.
Next we demonstrate how RA-NLD can be used to measure the effects of environment changes on the correlated non-robust features both visually and quantitatively.

\begin{definition}[\emph{Encountered set of states}]
Let $\Psi:\mathcal{S} \to \mathcal{S}$ be a function that transforms states $s \in \mathcal{S}$ of an MDP $\mathcal{M}$.
Let $S$ be the set of states encountered when utilizing policy $\pi$ in $\mathcal{M}$.
Then $S^\Psi$ is defined to be the set of states encountered when utilizing the policy $\pi\circ\Psi$ in $\mathcal{M}$ i.e. when the policy state observations are transformed via $\Psi$.
\end{definition}
In this setting, comparing $\mathcal{G}_S$ and $\mathcal{G}_{S^\Psi}$ will provide a qualitative picture of how the environmental change affects the learned vulnerable representation patterns.
In order to give a more quantitative metric for this change we define
\begin{definition}[\emph{Feature Correlation Quotient}]
  For two sets of states $S$ and $S'$, the feature correlation quotient is given by
  \[
  \Lambda(S',S) = \frac{\mathcal{G}_{S'}^{\top}\mathcal{L}(S)\mathcal{G}_{S'}}{\mathcal{G}_S^{\top}\mathcal{L}(S)\mathcal{G}_S}.
  \]
\end{definition}

\begin{algorithm}[t]
   \caption{RA-NLD: Robustness Analysis via Non-Lipschitz Directions in the Deep Neural Policy Manifold}
   \label{hmap}
\begin{algorithmic}
   \STATE {\bfseries Input:} MDP $\mathcal{M}$, state-action value function $Q(s,a)$, actions $a \in \mathcal{A}$, states $s \in \mathcal{S}$, the transition probability kernel $\mathcal{P}(s,a,s')$ 
    \STATE {\bfseries Output:} Principal non-Lipschitz direction $\mathcal{G}(i,j)$
   \FOR{$s =s_0$ {\bfseries to} $s_T$}
   \STATE $\tau(s,a) = \mathbbm{1}_{a = \argmax_{a'} Q(s,a')}(a)$
   \STATE $\pi(s_{\textrm{g}},a) = \softmax(Q(s_{\textrm{g}},a))$
   \STATE $J(s,s_g)$ = $- \sum_{a \in \mathcal{A}} \tau(s,a)\log(\pi(s_g,a))$
    \STATE $\mathcal{L}$ += $\nabla_{s_g} J(s,s_g)\rvert_{s_g = s} [\nabla_{s_g} J(s,s_g)\rvert_{s_g = s}]^{\top}$
   \ENDFOR
   \STATE {\bfseries Return:} Eigenvector $\mathcal{G}$ corresponding to largest eigenvalue of $\mathcal{L}$
\end{algorithmic}
\end{algorithm}

\begin{proposition}[\emph{Boundedness of Feature Correlation Quotient}]
  For any two sets of states $S$ and $S'$ it holds that $0 \leq \Lambda(S',S) \leq 1.$
\end{proposition}
\begin{proof}
  By Proposition \ref{prop:eigenvector},
  \[
  \mathcal{G}_{S'}^{\top}\mathcal{L}(S)\mathcal{G}_{S'} \leq \max_{\lVert z \rVert_2 =1} z^{\top}\mathcal{L}(S)z
    = \mathcal{G}_{S}^{\top}\mathcal{L}(S)\mathcal{G}_{S}
  \]
Thus the numerator of $\Lambda(S',S)$ is always less than or equal to the denominator i.e. $\Lambda(S',S) \leq 1$. Furthermore, $\mathcal{L}(S)$ is positive semidefinite, as it is a sum of rank one projection matrices, and hence $\Lambda(S',S)\geq 0$.
\end{proof}

\begin{figure*}[t]
\footnotesize
\centering
\stackunder[4pt]{\includegraphics[scale=0.168]{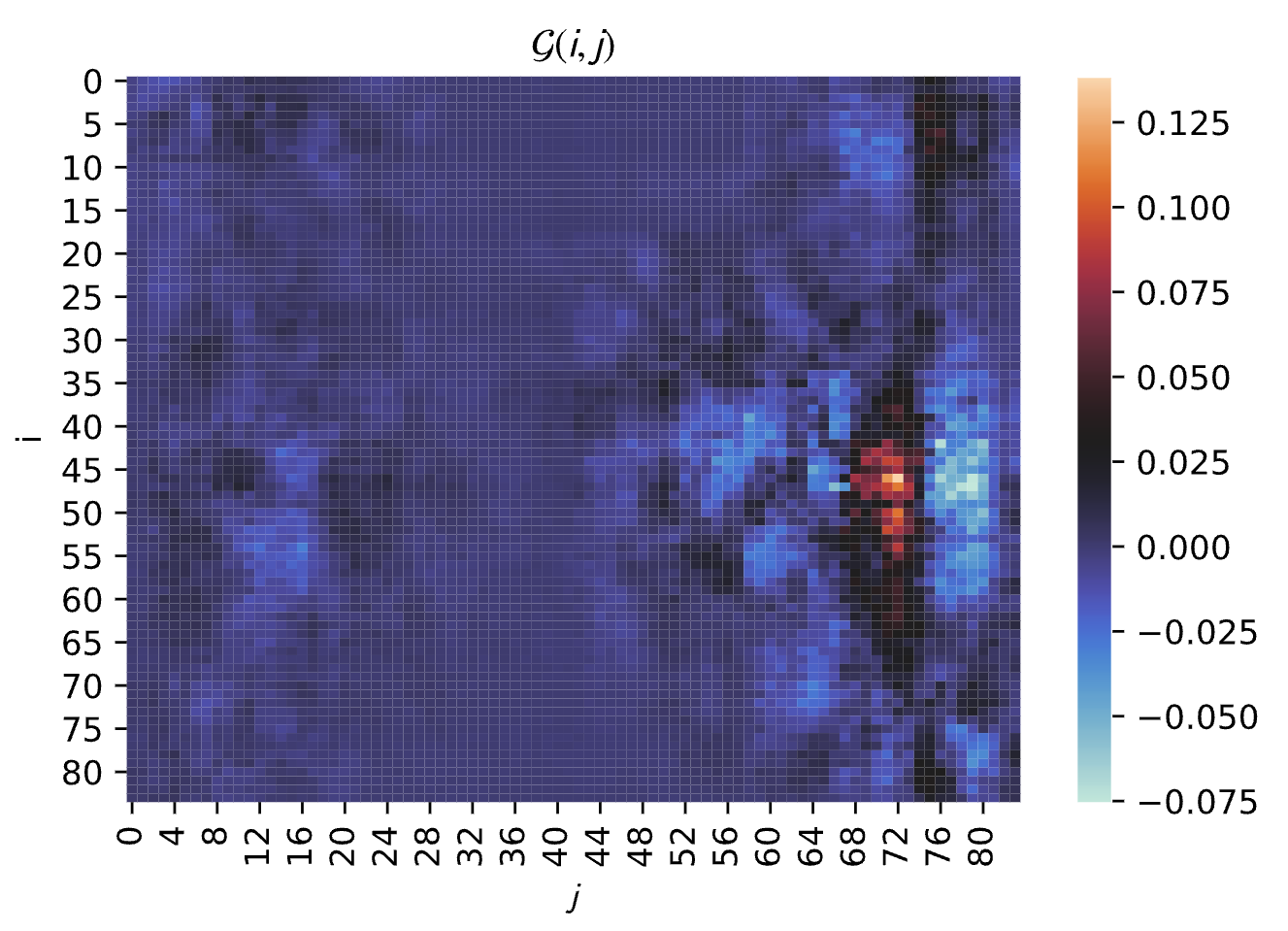}}{}
\hskip 0.1pt
\stackunder[4pt]{\includegraphics[scale=0.158]{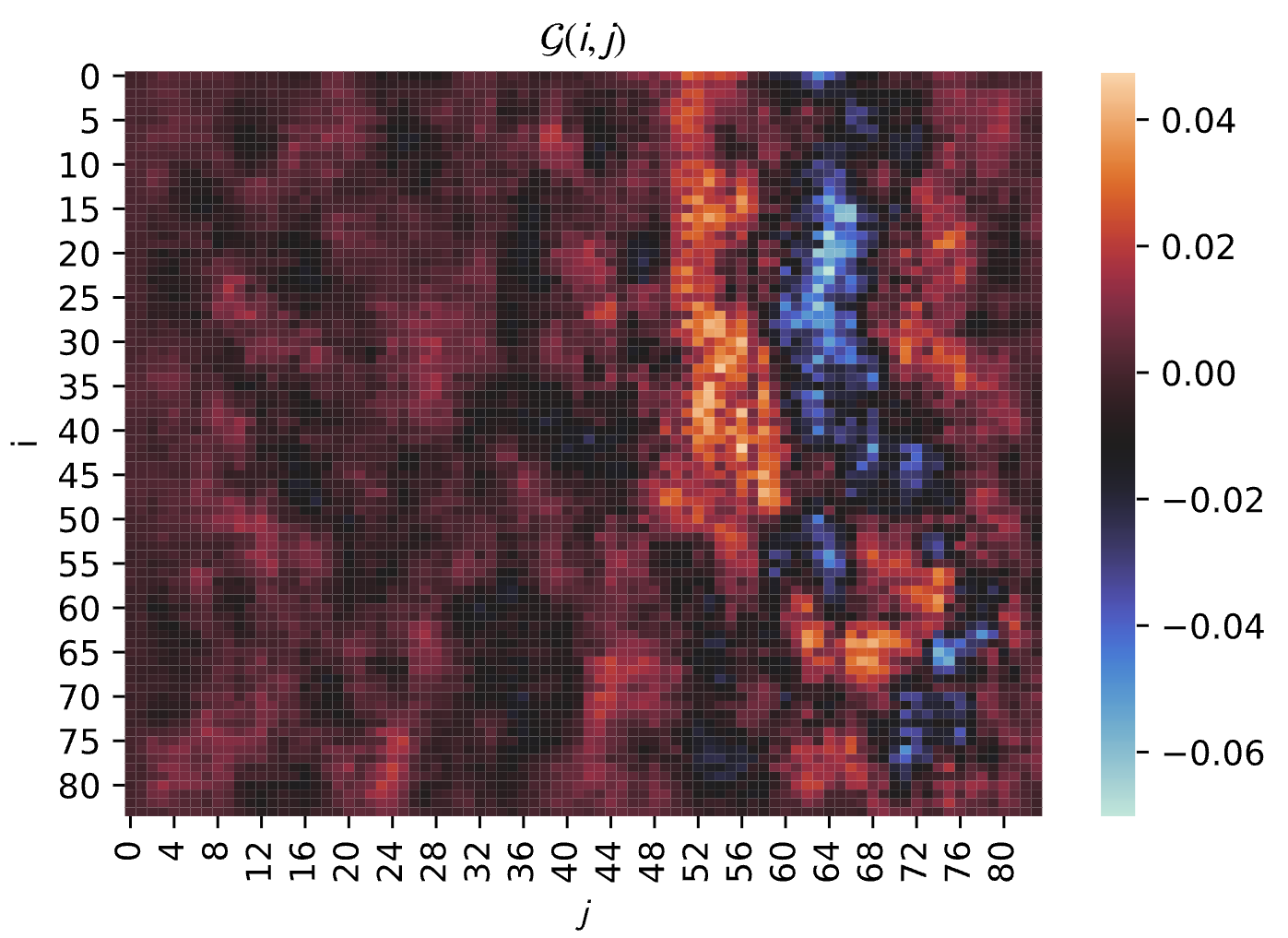}}{}
\stackunder[4pt]{\includegraphics[scale=0.152]{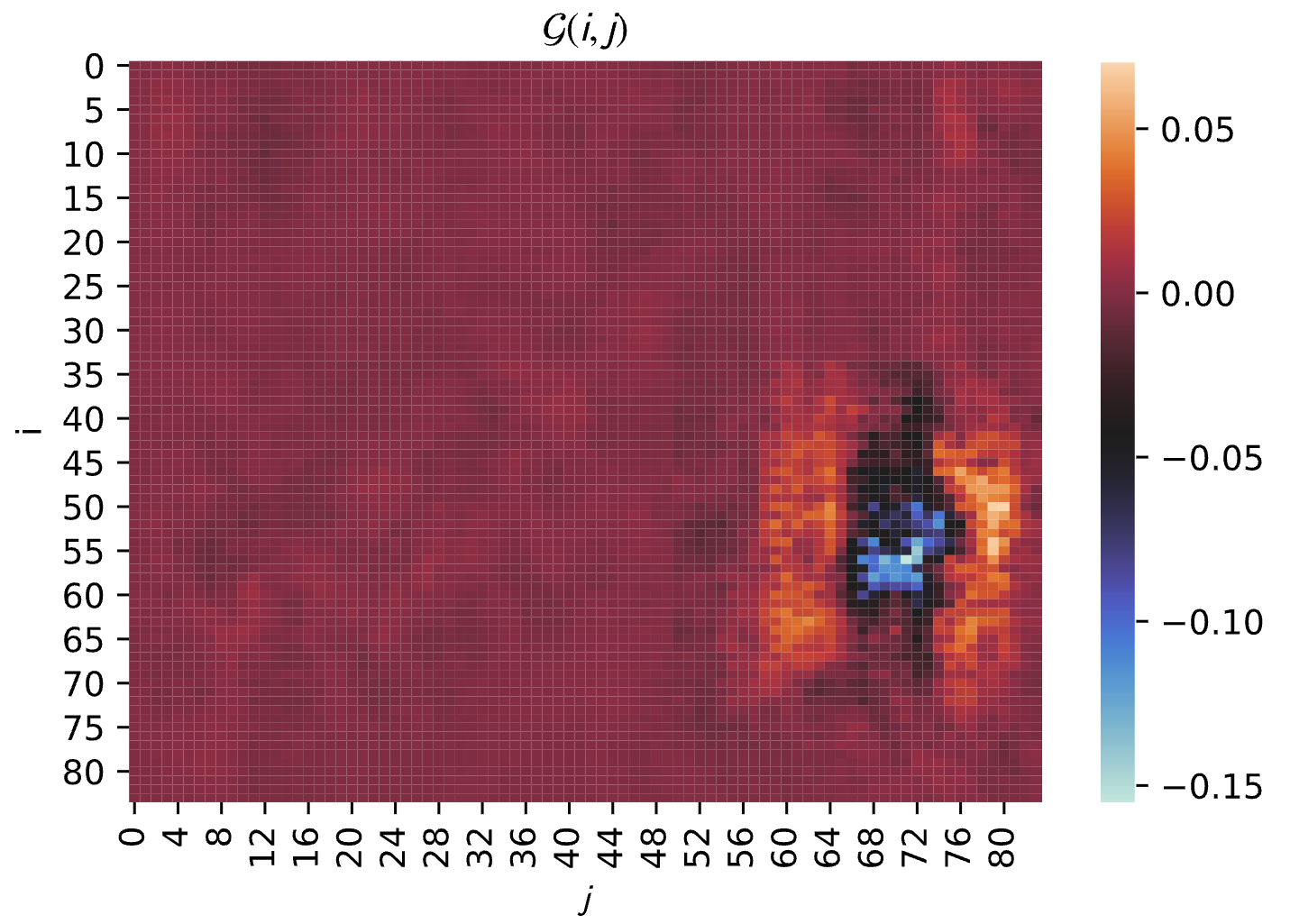}}{}
\stackunder[4pt]{\includegraphics[scale=0.158]{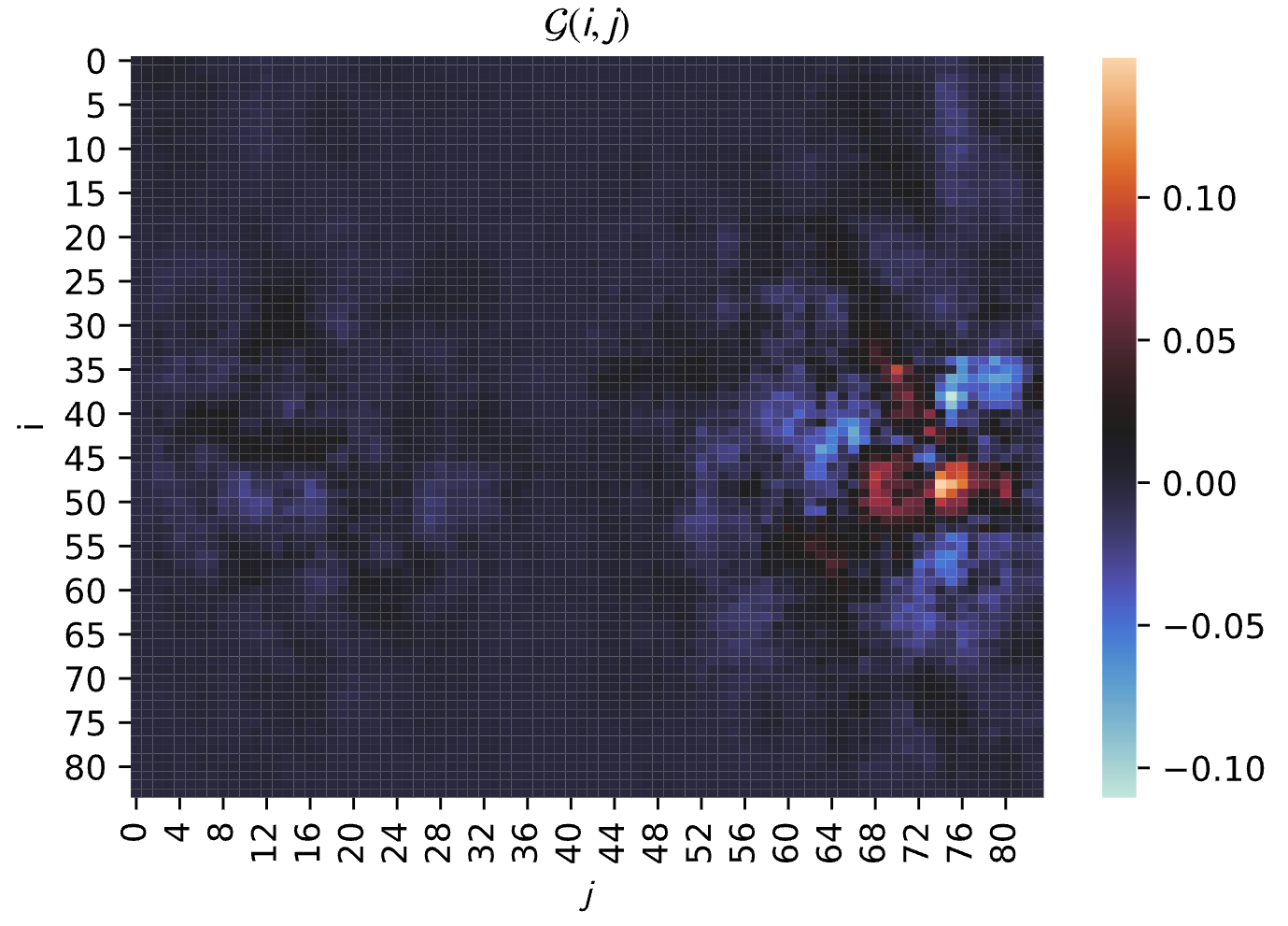}}{} \\
\vskip-0.12in
\stackunder[4pt]{\includegraphics[scale=0.16]{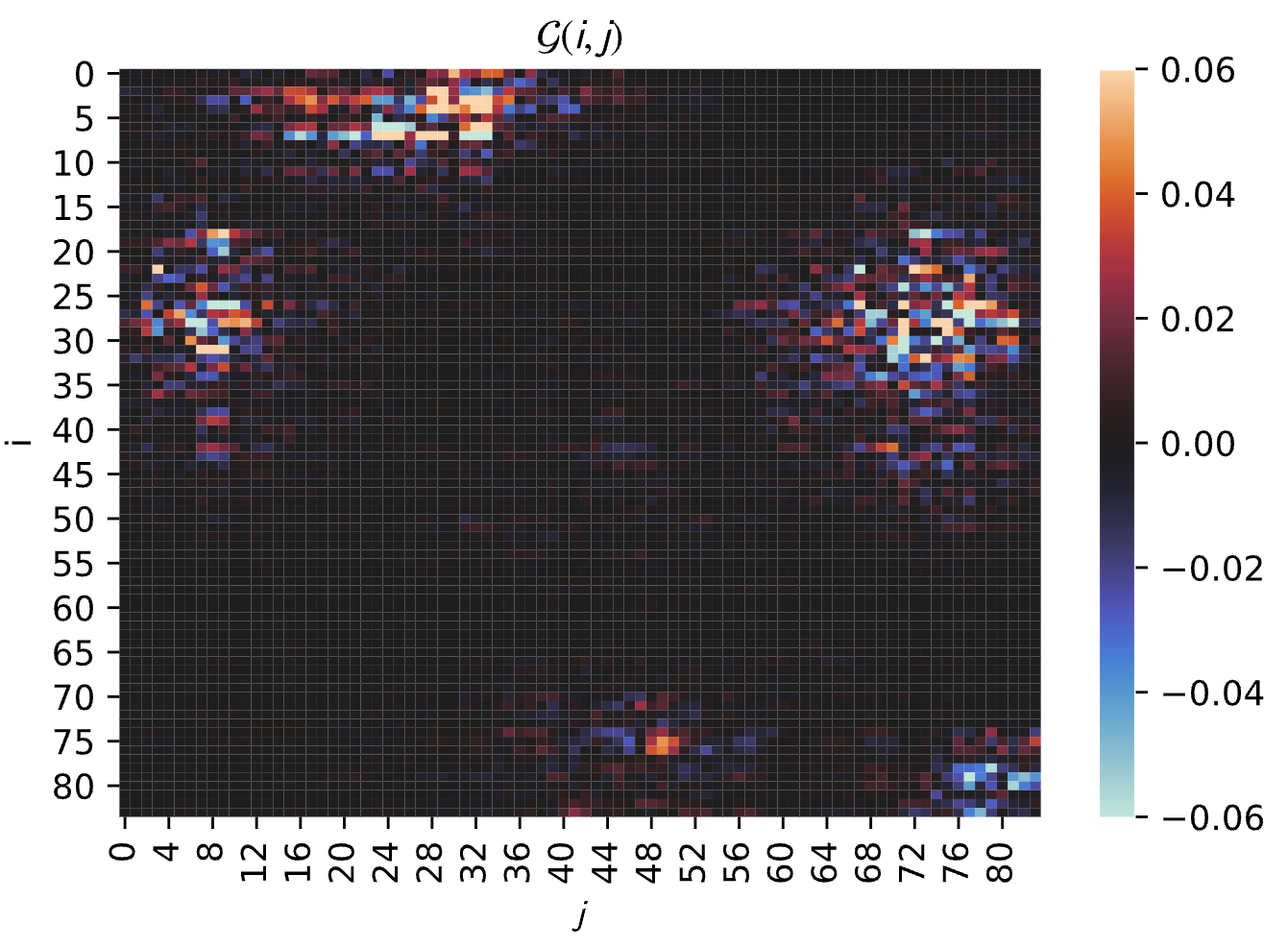}}{}
\stackunder[4pt]{\includegraphics[scale=0.16]{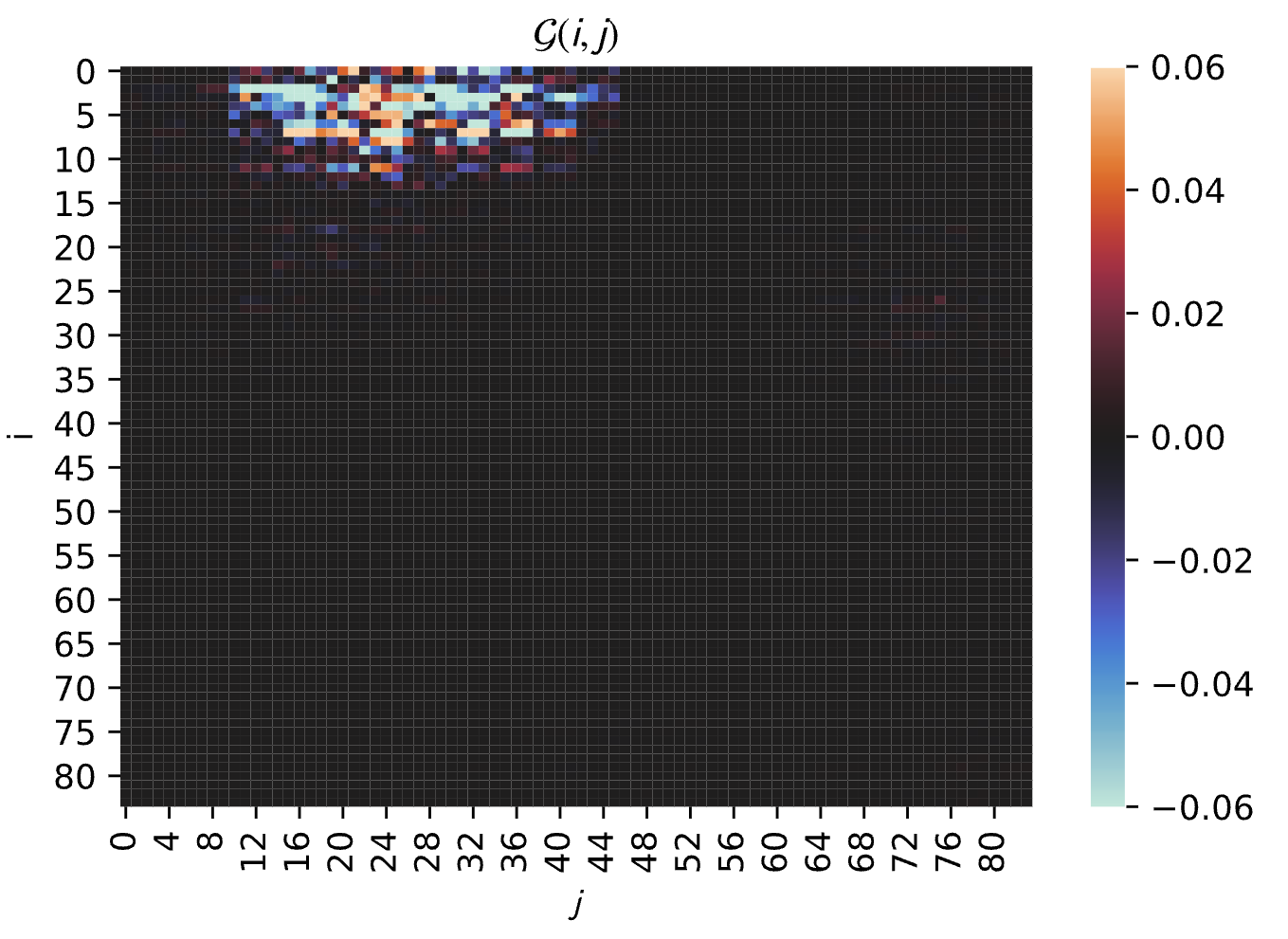}}{}
\stackunder[4pt]{\includegraphics[scale=0.16]{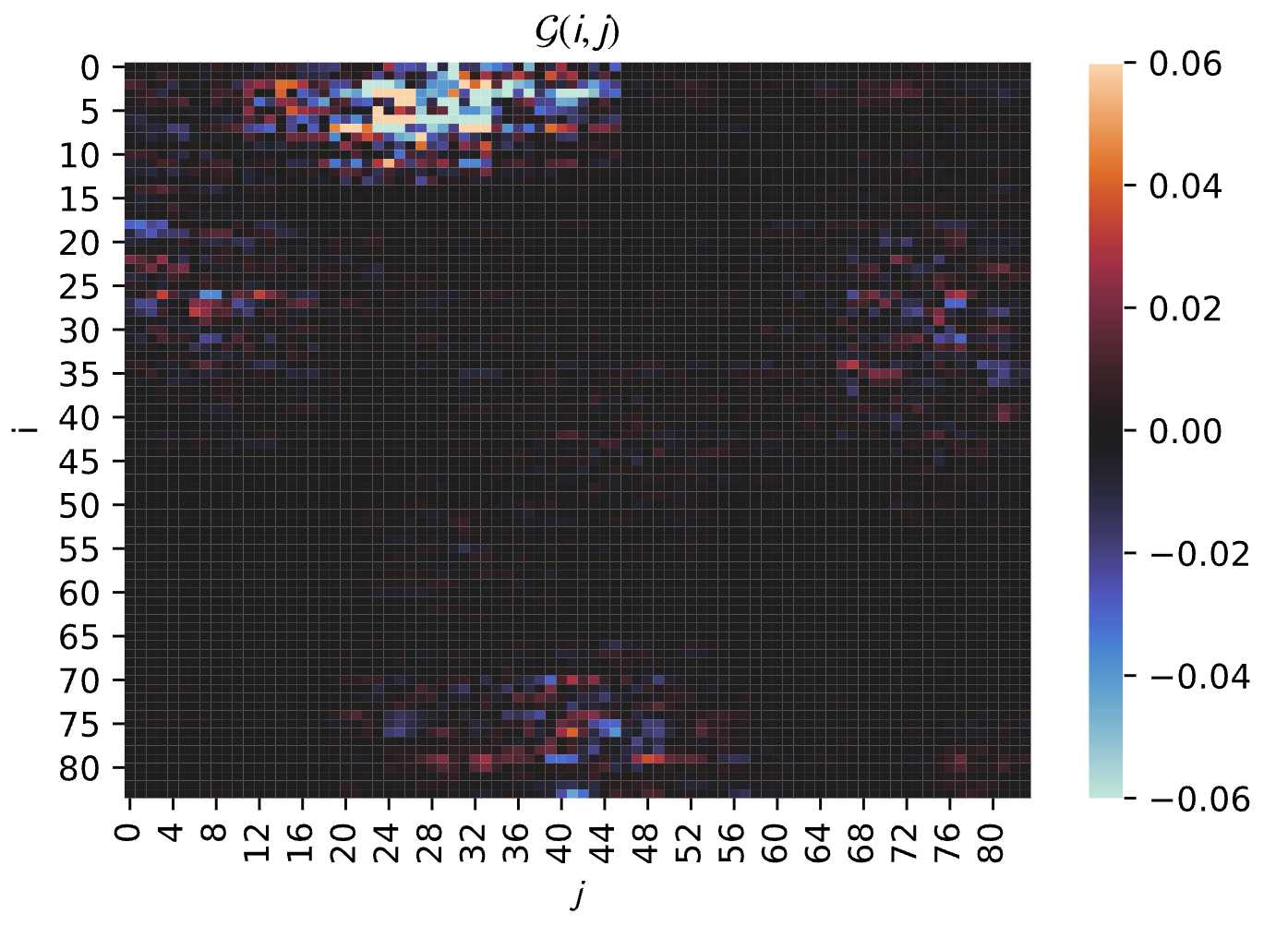}}{}
\stackunder[4pt]{\includegraphics[scale=0.16]{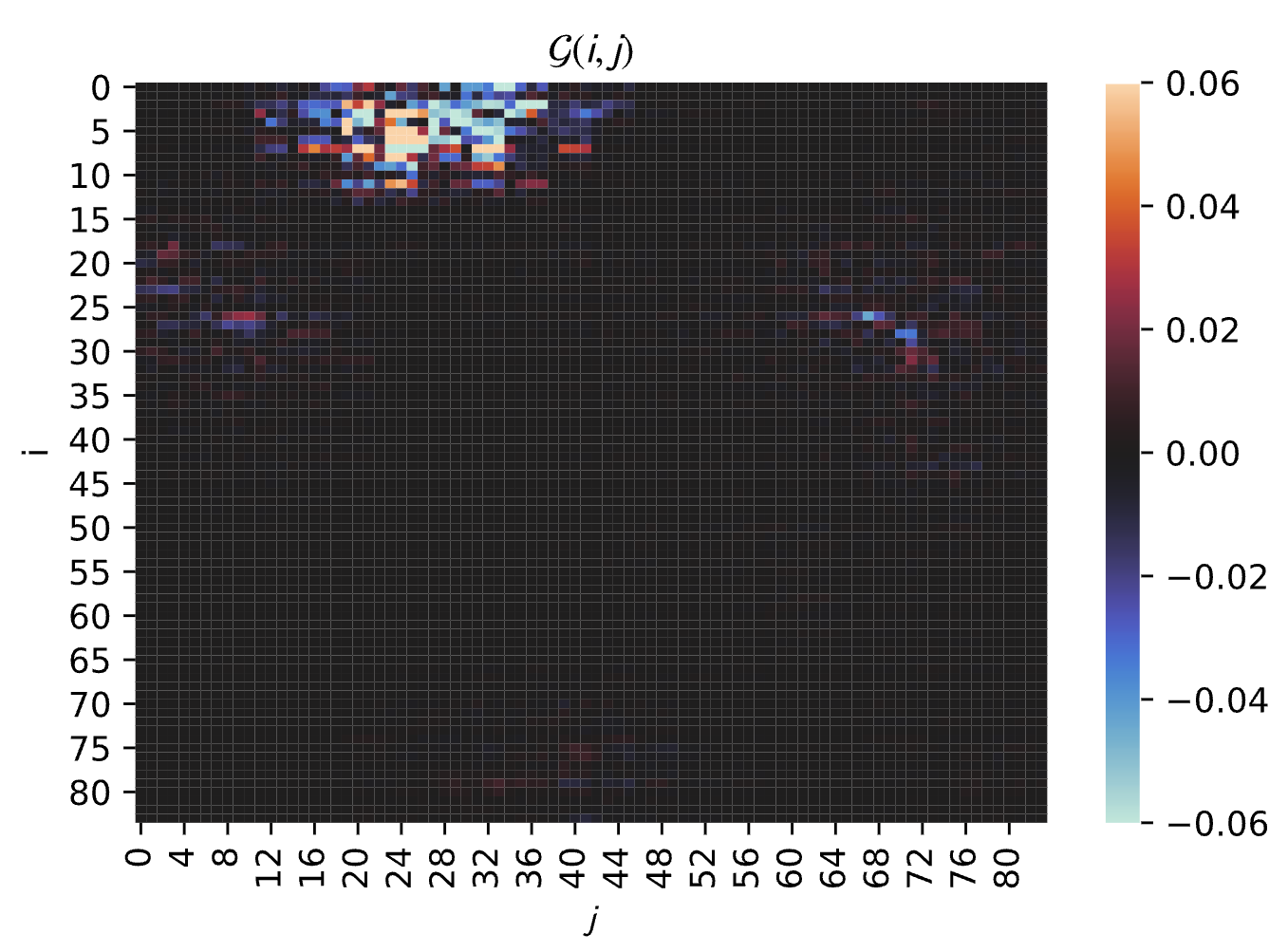}}{}
\vskip -0.16in
\caption{RA-NLD results of untransformed states and states under adversarial perturbations computed via Carlini\&Wagner, Nesterov Momentum, and elastic-net regularization for Pong and BankHeist. Row1: Pong. Row2: BankHeist. Column1: Untransformed. Column2: C\&W. Column3: Nesterov Momentum. Column4: Elastic-Net}
\label{gmapattack}
\vskip -0.12in
\end{figure*}

\begin{table*}[t!]
\vskip -0.06in
\caption{The feature correlation quotient $\Lambda(\hat{S},S)$ and $\Lambda(S^\Psi,S)$ for the adversarial transformations: Carlini\&Wagner, Nesterov Momentum, DeepFool, Elastic-Net.}
\vskip -0.01in
\label{quoadv}
\centering
\scalebox{0.98}{
\begin{tabular}{lcccccr}
\toprule
Environments              &  Base Observations
						    				  &  Carlini\&Wagner
						    				  &  Nesterov Momentum
						              &  DeepFool
						    				  &  Elastic-Net \\
\midrule
Freeway     & 0.9917$\pm$0.0023    & 0.9499$\pm$0.02056   & 0.7868$\pm$0.02162   & 0.6869$\pm$0.02981      & 0.72590$\pm$0.0592       \\
BankHeist   & 0.8360$\pm$0.0116    & 0.2837$\pm$0.02316   & 0.3407$\pm$0.02412   & 0.1748$\pm$0.04421      & 0.30917$\pm$ 0.0521       \\
RoadRunner  & 0.7652$\pm$0.0385    & 0.1621$\pm$0.02199   & 0.3826$\pm$0.03118   & 0.5353$\pm$0.03127      & 0.52506$\pm$ 0.0782       \\
Pong 		  	& 0.4934$\pm$0.0391    & 0.0408$\pm$0.04056   & 0.3444$\pm$0.01981   & 0.3277$\pm$0.02871      & 0.10529$\pm$ 0.0629       \\
\bottomrule
\end{tabular}
}
\end{table*}

Therefore, the feature correlation quotient $\Lambda(S',S)$ is a number between zero and one which intuitively measures how correlated the non-robust features from $S'$  are to those from $S$.
When measuring how an environmental change affects the decisions made by the deep neural policy and the non-robust representations learnt, it is also important to take the stochastic nature of the MDP into account. In particular, the non-robust features observed with two different executions of the same policy may differ slightly due to the inherent randomness of the MDP.
To account for this, we first collect a baseline set of states with no modification $S$. We then collect a set of states $\hat{S}$ with no modification, and $S^\Psi$ with modification. By comparing $\Lambda(\hat{S},S)$ to $\Lambda(S^\Psi,S)$ we can see how much of the decrease in average correlation is caused by the stochastic nature of the MDP, and how much of the decrease is caused by the environmental change.

\section{Experimental Analysis}

The deep reinforcement learning policies evaluated in our experiments are trained with the Double Deep Q-Network algorithm \cite{hado16} initially proposed in \cite{hado10} with the architecture proposed by \citet{wang16}, and State-Adversarial Double Deep Q-Network (see Section \ref{advrl}) with experience replay \cite{tom16}. The set of states $S$ is collected over 10 episodes. We use the adversarial methodology from \citet{korkmaz2023icml}. The adversarial perturbation hyperparameters are: for the Carlini\&Wagner formulation $\kappa$ is $10$, learning rate is $0.01$, initial constant is 10, for the elastic-net regularization formulation $\beta$ is $0.0001$, learning rate is $0.1$, maximum iteration is $300$, for Nesterov Momentum $\epsilon$ is $0.001$, and decay factor is $0.1$.\footnote{The hyperparameters for the adversarial attacks are fixed to the same levels as base studies to provide transparency and consistency with the prior work. Furthermore, note that the setting is also optimized to achieve the most effective adversarial perturbations (i.e. perturbations causing the largest decrease on the discounted expected cumulative rewards obtained by the policy).}

\setcounter{footnote}{1}

\subsection{Non-Robust Feature Shifts under Adversarial Perturbations}
In this section we investigate the effects of adversarial attacks on the learnt correlated non-robust features. Figure \ref{gmapattack} reports the RA-NLD results for the untransformed states and the adversarially attacked state observations. In particular, these perturbations are computed via the Nesterov momentum, Carlini\&Wagner, and elastic-net regularization formulations (see Section \ref{advpert}). 
Figure \ref{gmapattack} demonstrates that different adversarial formulations surface different sets of correlated non-robust features. Depending on the perturbation type, the correlated directions of instability can change quite noticeably. In fact, while the Carlini\&Wagner formulation leaves a distinct signature on the vulnerable representation pattern, the non-robust features under
\begin{figure*}[t]
\centering
\footnotesize
\stackunder[6pt]{\includegraphics[scale=0.2]{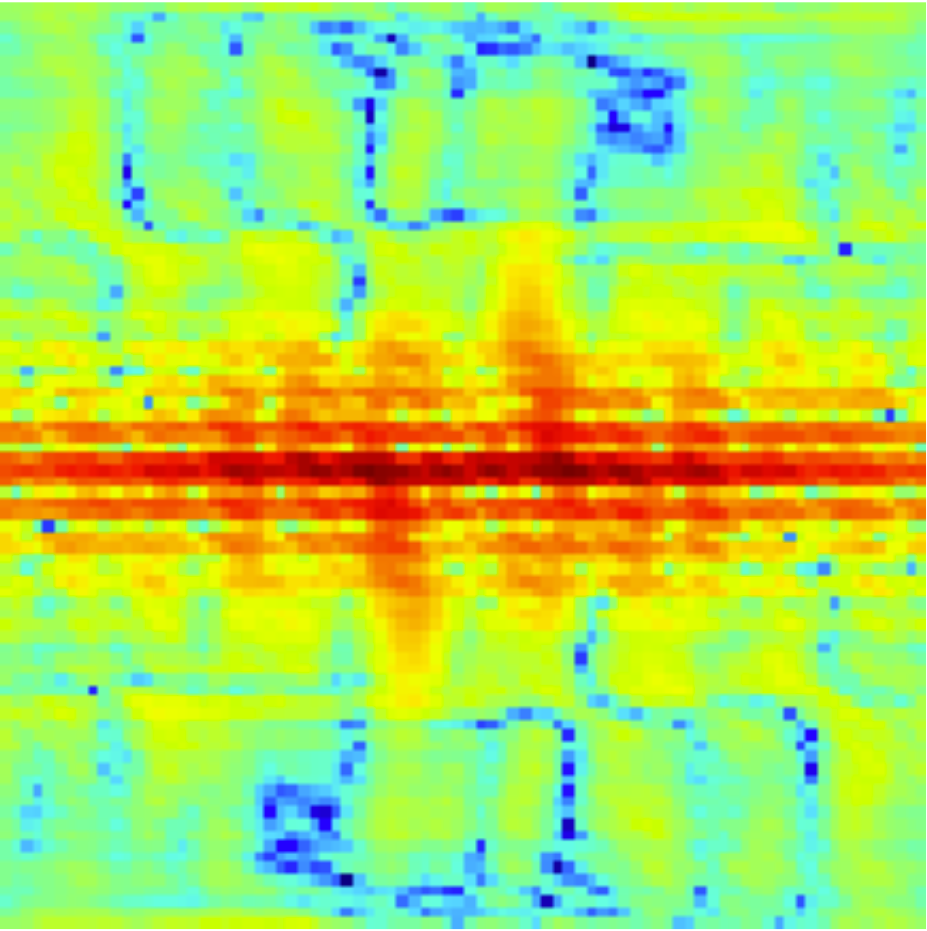}}{}
\hskip 0.1pt
\stackunder[6pt]{\includegraphics[scale=0.2]{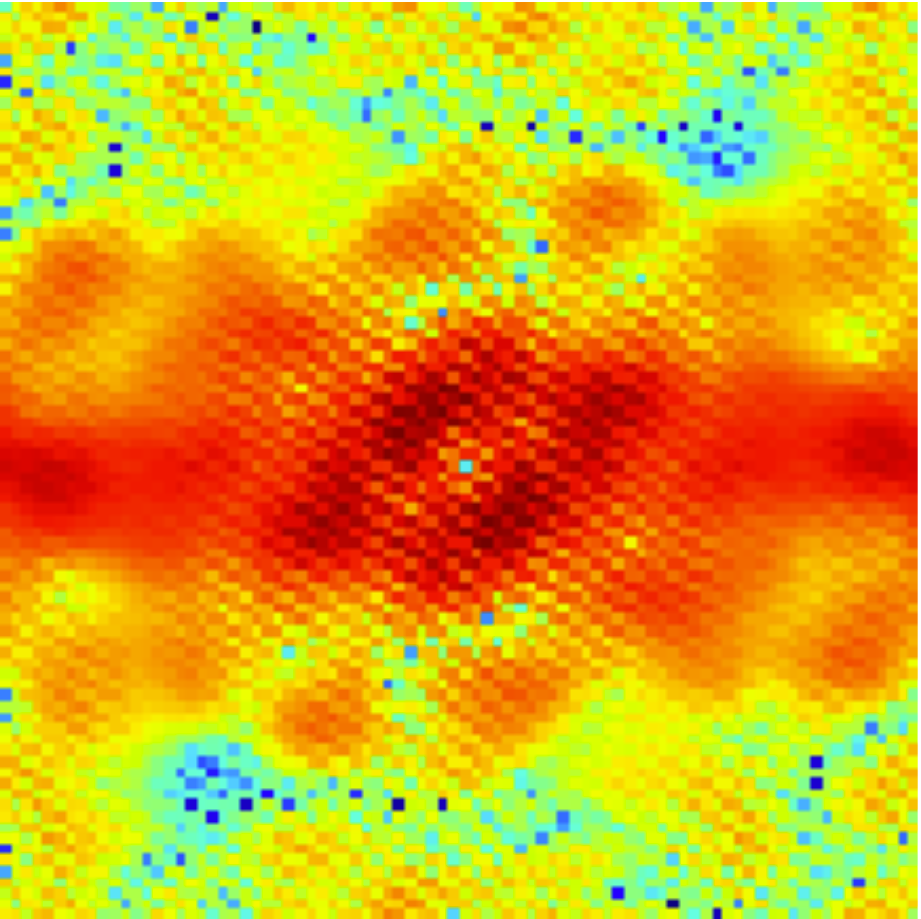}}{}
\stackunder[6pt]{\includegraphics[scale=0.2]{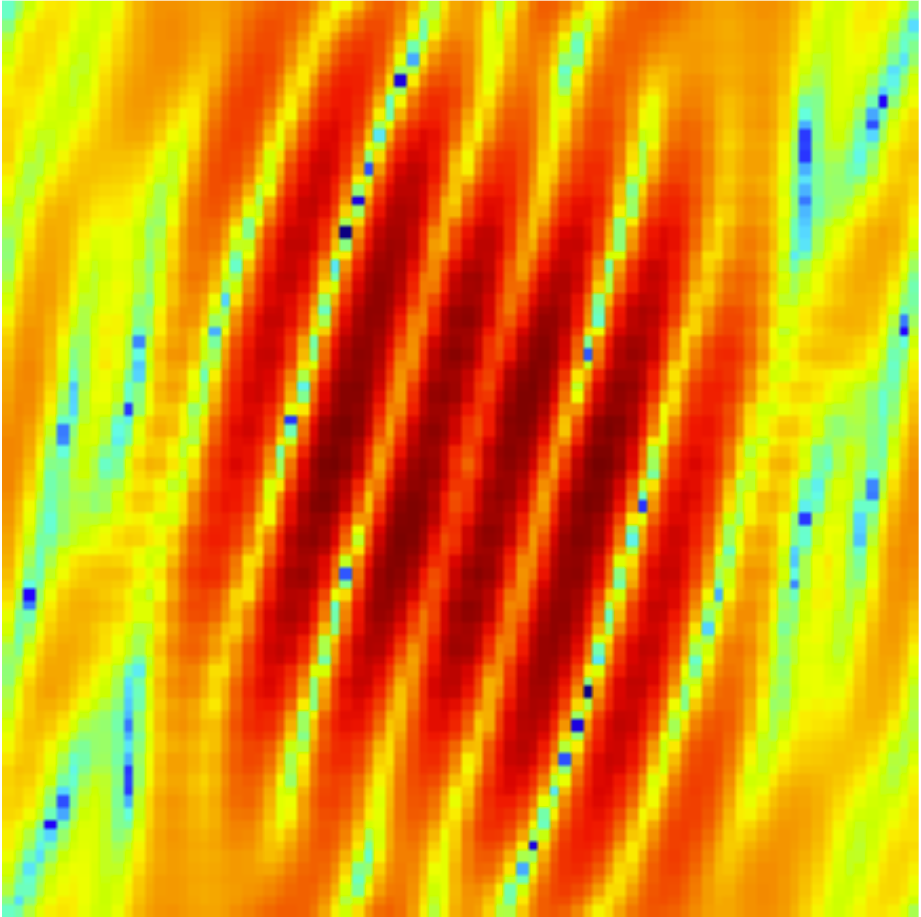}}{}
\stackunder[6pt]{\includegraphics[scale=0.2]{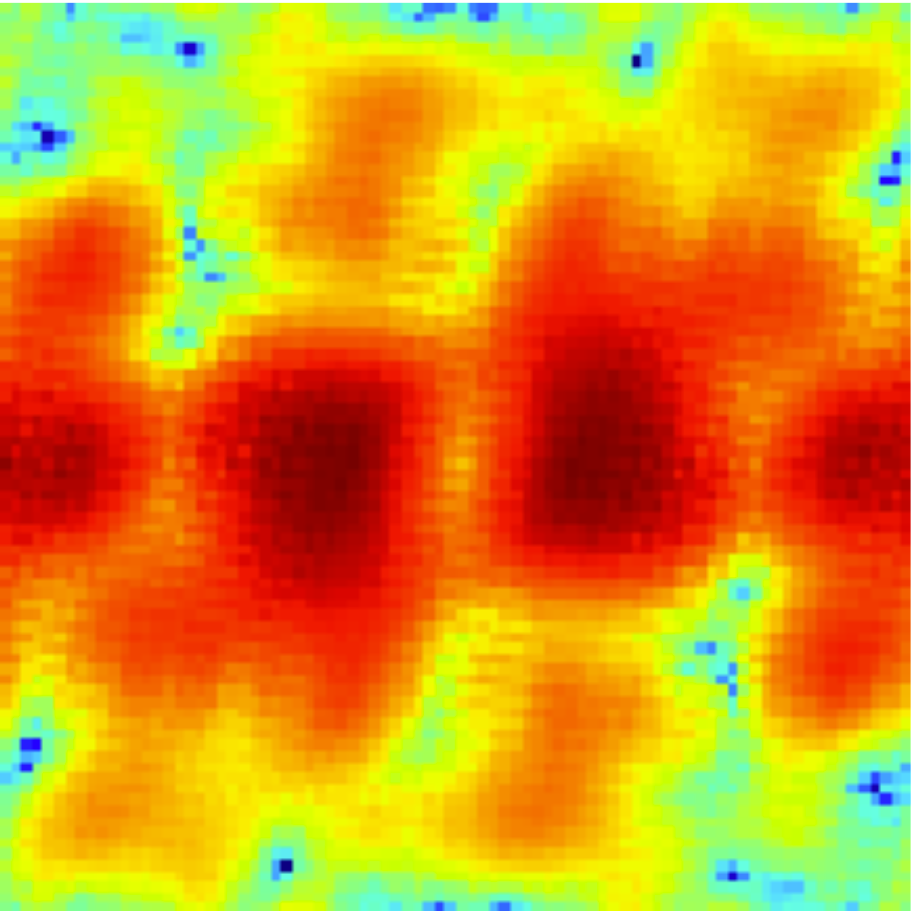}}{}\\
\vskip -0.08in
\stackunder[4pt]{\includegraphics[scale=0.2]{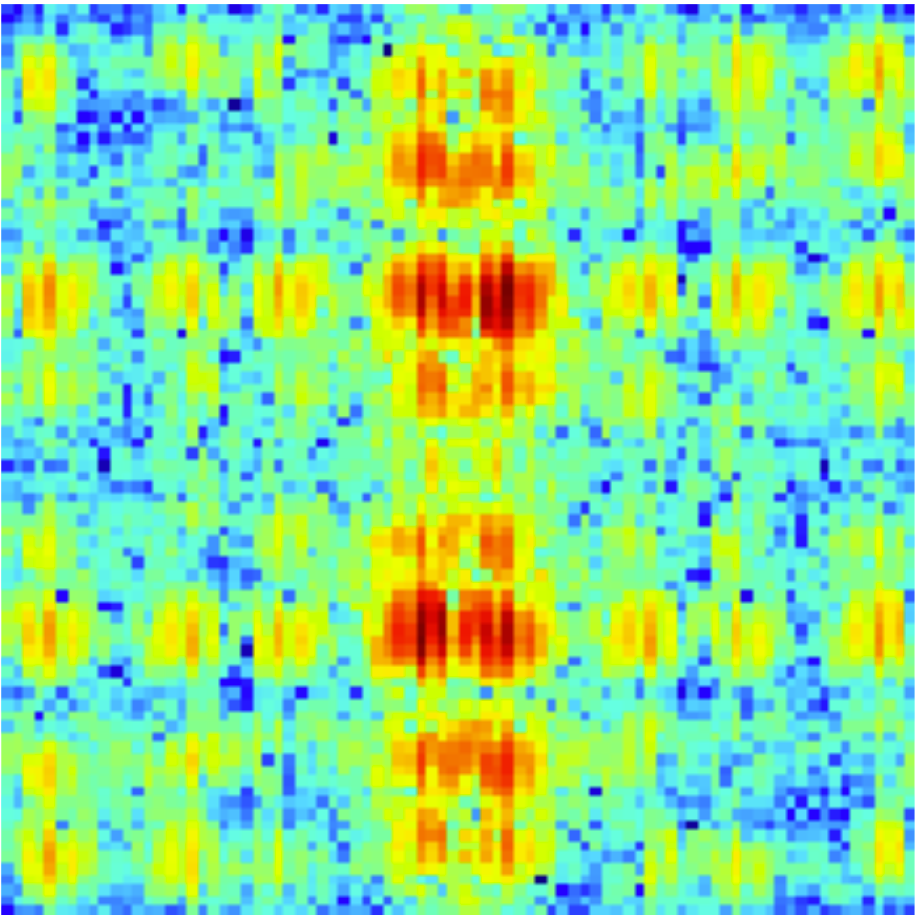}}{}
\hskip 0.1pt
\stackunder[4pt]{\includegraphics[scale=0.2]{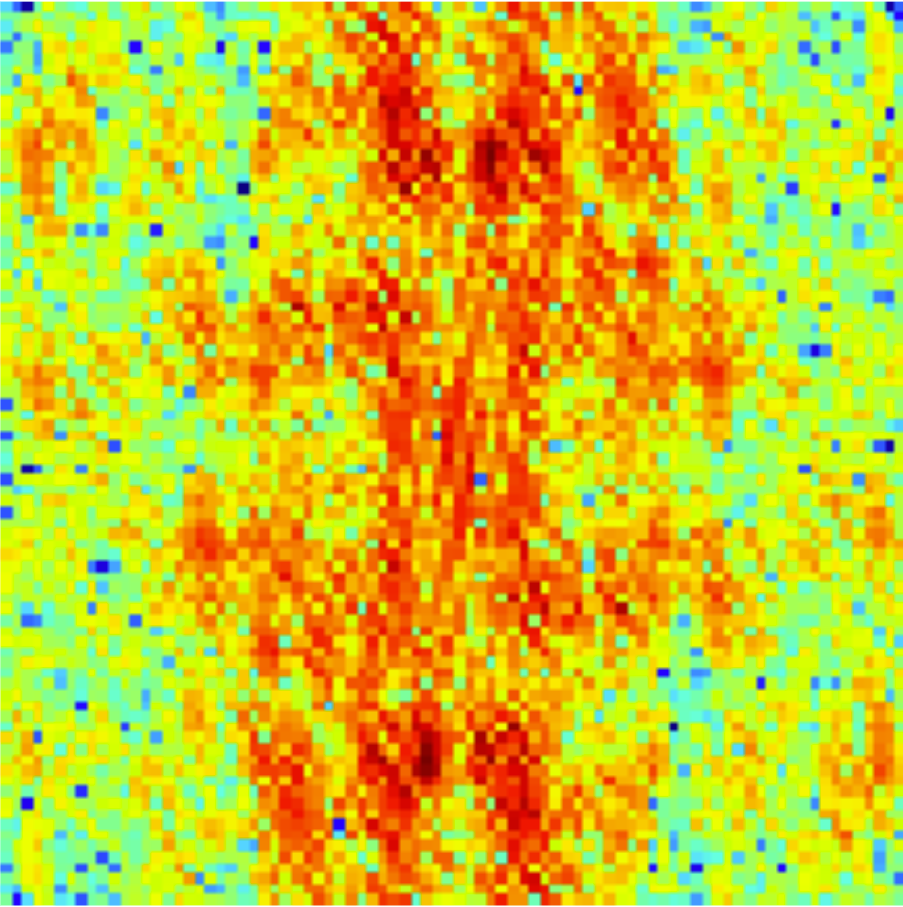}}{}
\stackunder[4pt]{\includegraphics[scale=0.194]{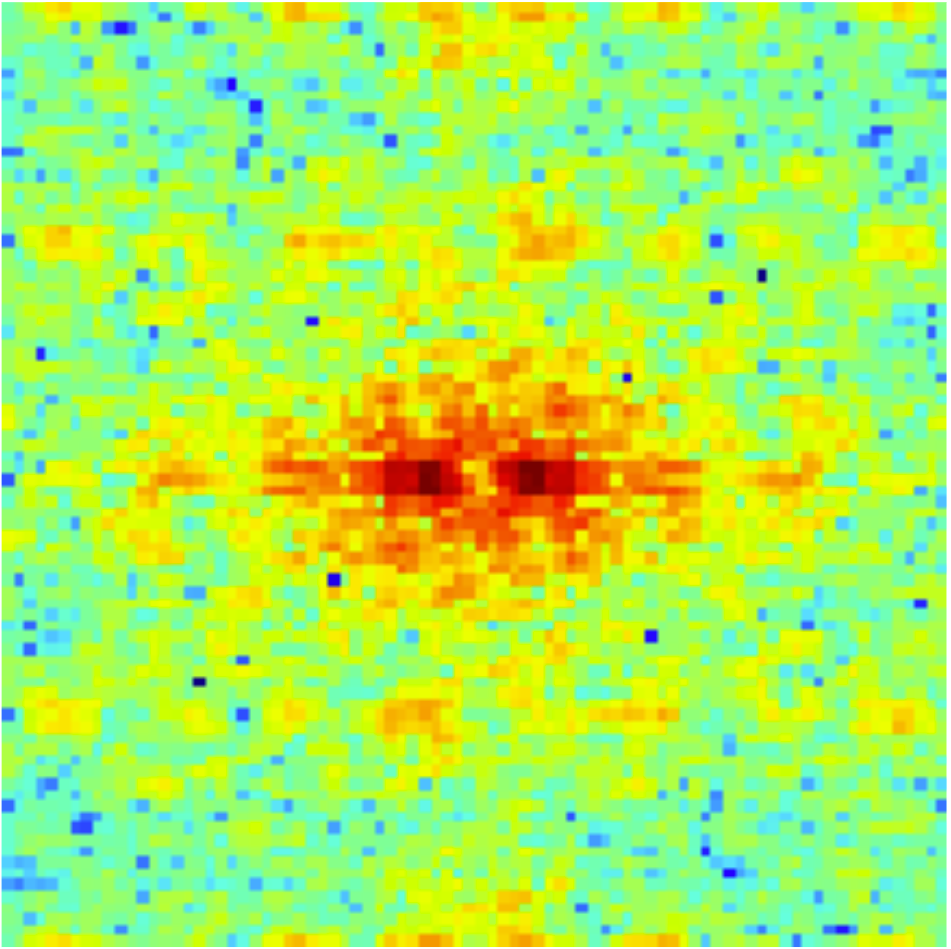}}{}
\stackunder[4pt]{\includegraphics[scale=0.2]{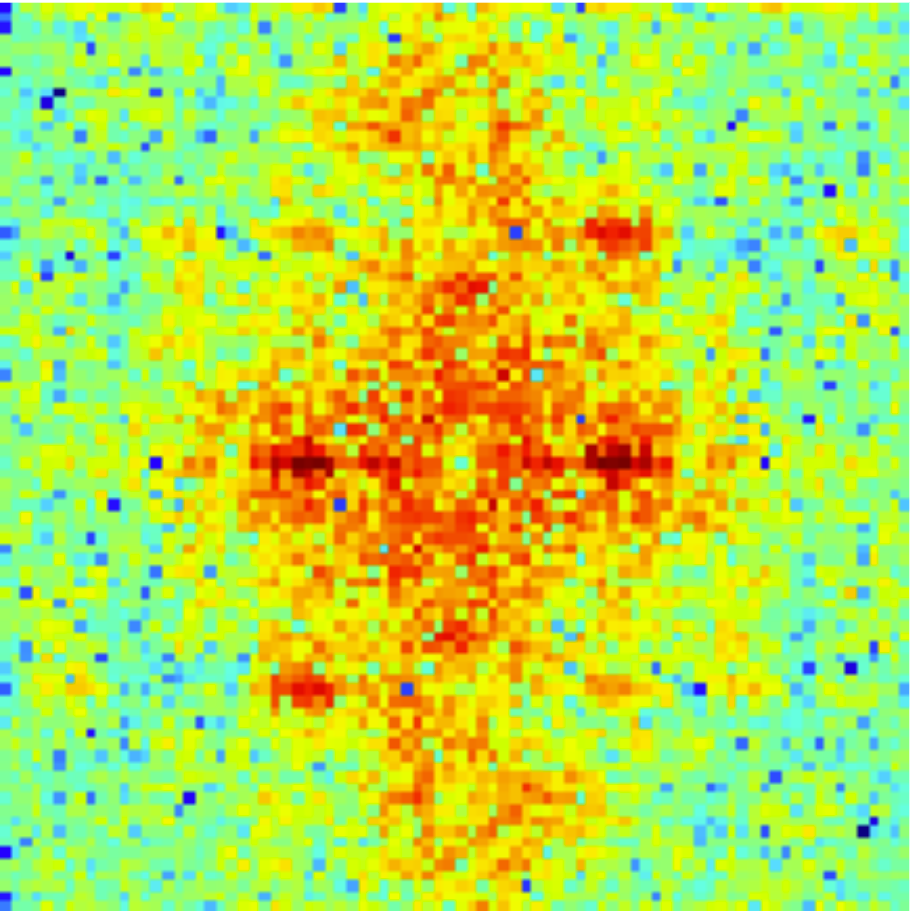}}{}
\vskip -0.1in
\caption{Fourier spectrum of the RA-NLD of the state-of-the-art adversarially and vanilla trained deep neural policies.\footnotemark Row1: Adversarial. Row2: Vanilla. Column1: RoadRunner. Column2: BankHeist. Column3: Pong. Column4: Freeway}
\label{advft}
\vskip -0.04in
\end{figure*}
\begin{figure}[t!]
\vskip -0.07in
\centering
\footnotesize
\stackunder[6pt]{\includegraphics[scale=0.099]{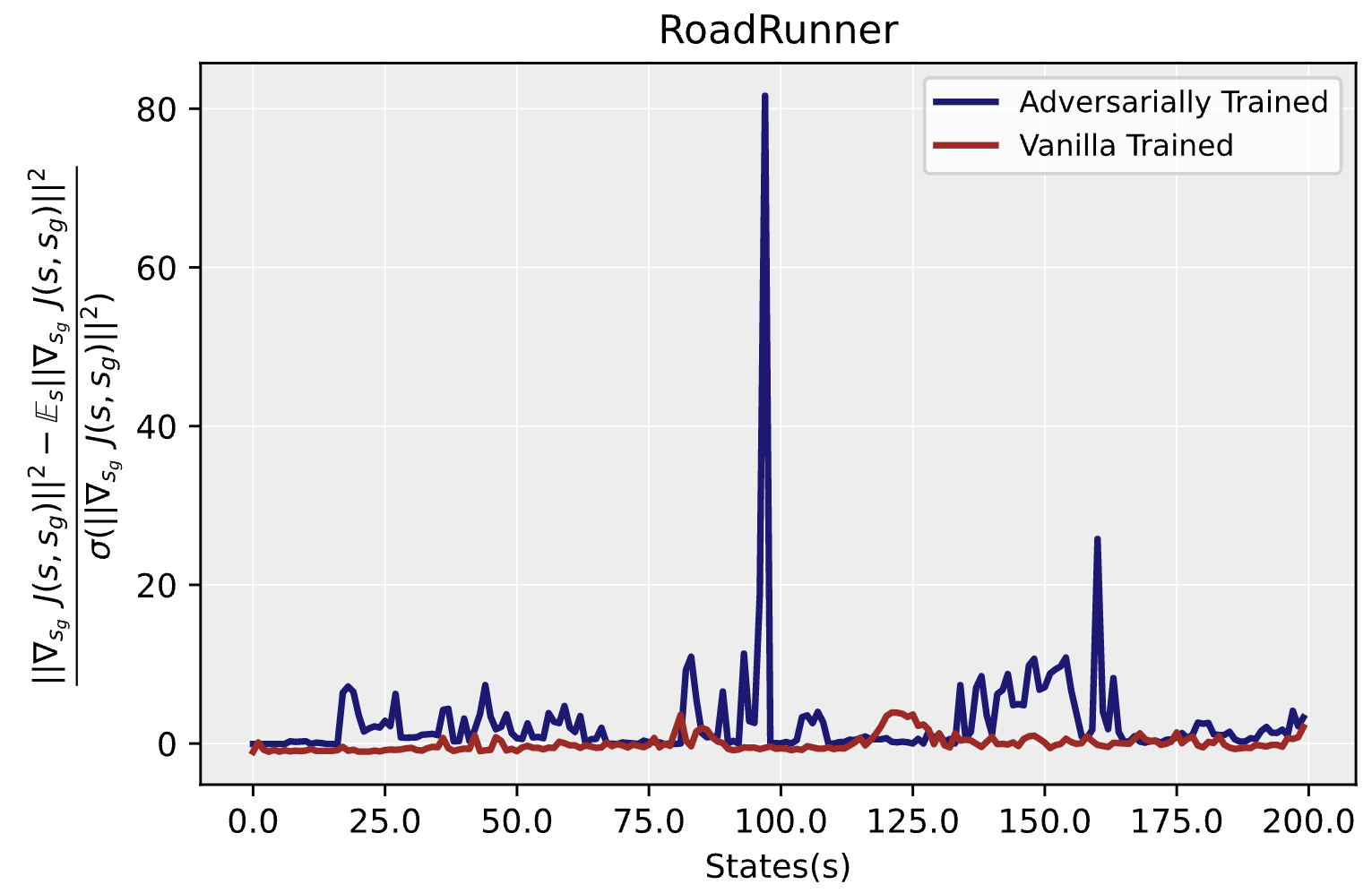}}{}
\hskip 0.1pt
\stackunder[6pt]{\includegraphics[scale=0.099]{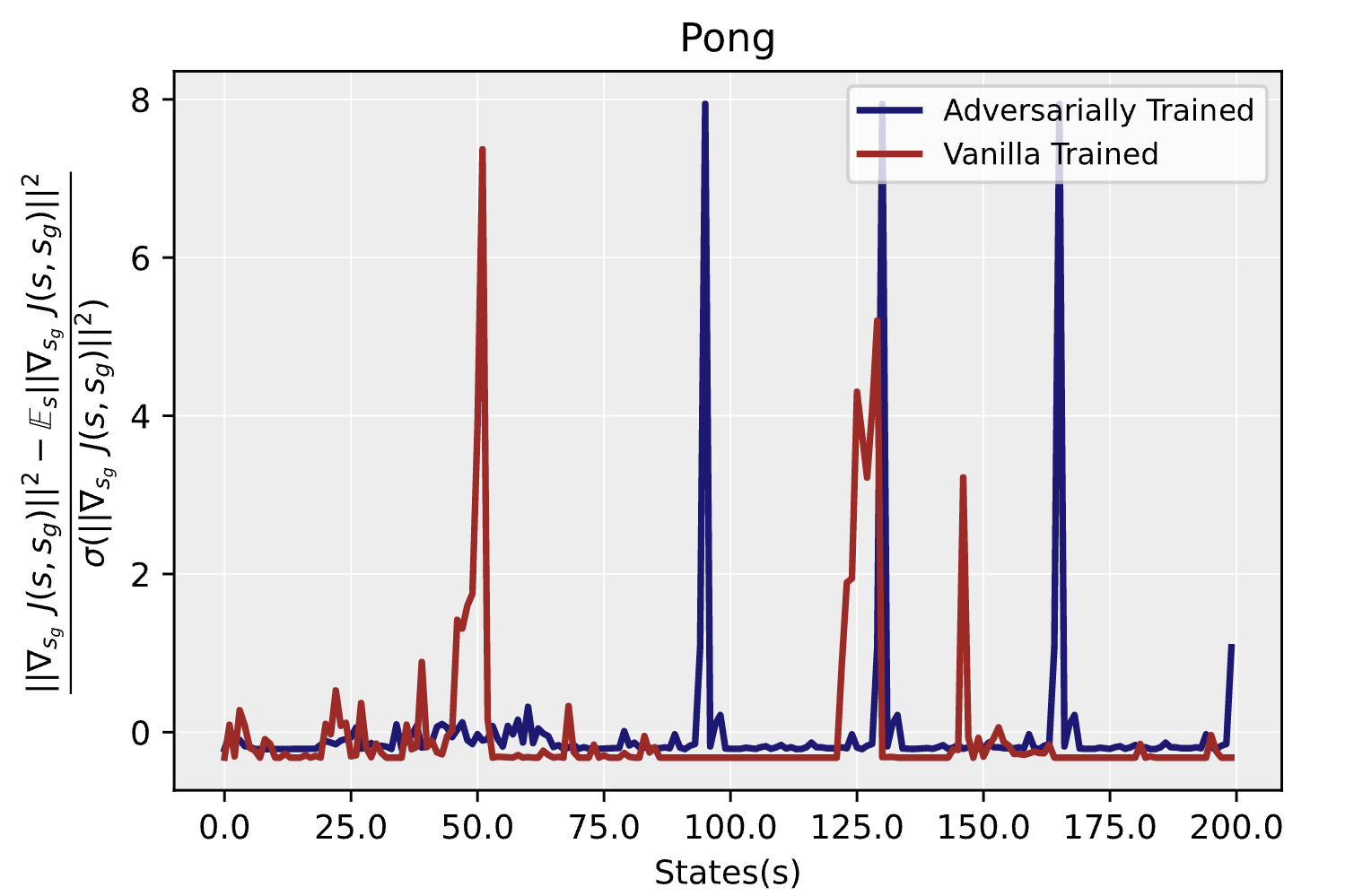}}{}
\stackunder[6pt]{\includegraphics[scale=0.099]{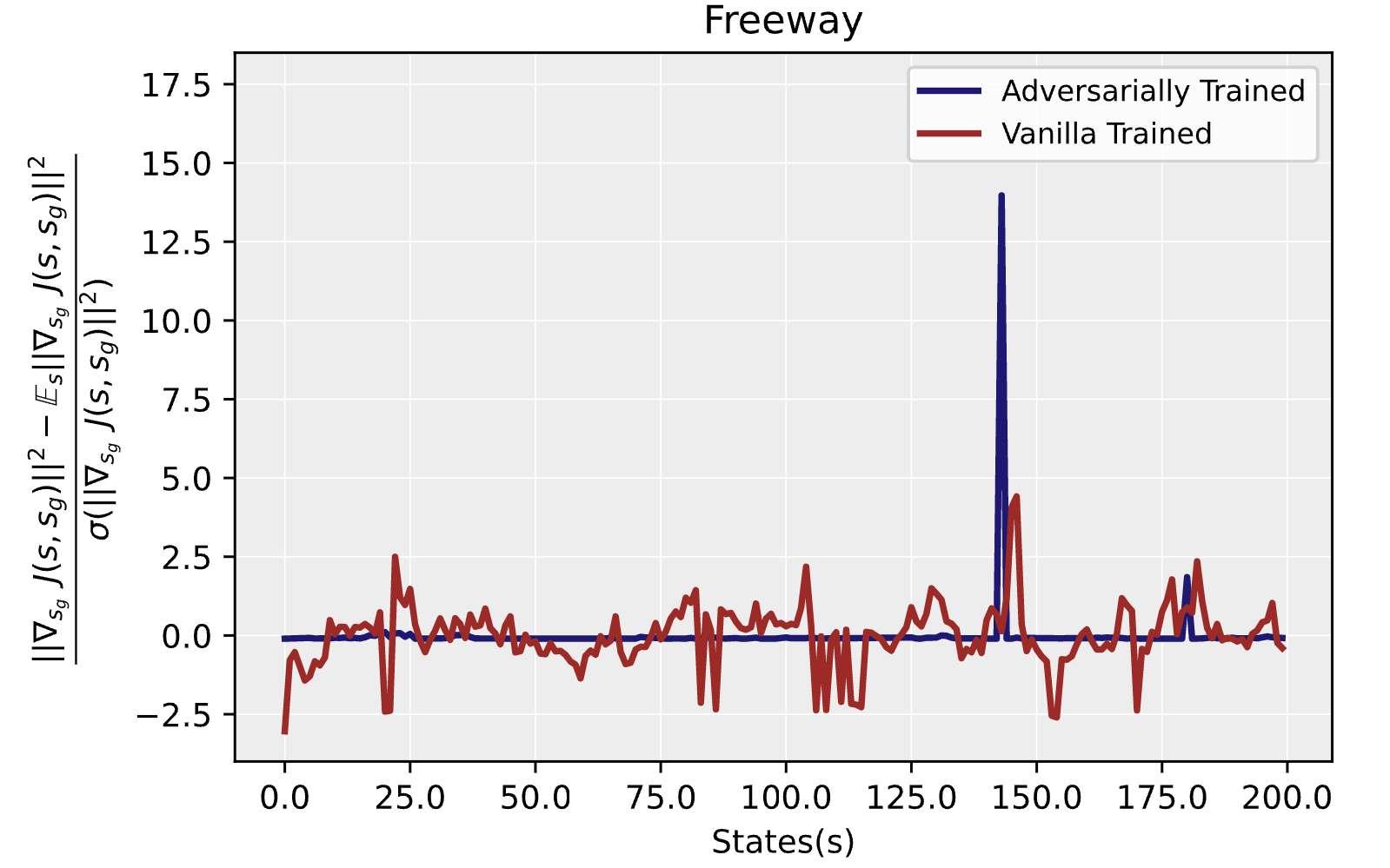}}{}
\vskip -0.12in
\caption{Standardized gradients $\lVert \nabla_{s_g} J(s_i,s_g) \rVert^2$ for vanilla trained and state-of-the-art certified adversarially trained deep reinforcement learning policies.}
\label{advtrace}
\vskip -0.18in
\end{figure}
Nesterov momentum appear most similar to those of the untransformed states. 
Thus, evidently our imaging technique helps to understand the rationale behind policy decision making and the vulnerabilities of deep reinforcement learning policies by allowing us to visualize precisely how non-robust features change with different sets of specifically optimized adversarial directions.
Table \ref{quoadv} reports the feature correlation quotient $\Lambda(\hat{S},S)$ and $\Lambda(S^\Psi,S)$ results where $S$ consists of untransformed states and $S^\Psi$ consists of states modified by the Nesterov Momentum, Carlini\&Wagner, elastic-net regularization and DeepFool formulations respectively.
Note that in all games the setting where $\hat{S}$ consists of a set of untransformed states from an independent execution has
the highest feature correlation quotient $\Lambda(\hat{S},S)$.
Therefore the additional decrease of $\Lambda(S^\Psi,S)$ when $S^\Psi$ is modified by adversarial perturbations can be attributed to changes in non-robust features caused by the perturbations.
Observe also that the qualitative similarity between the visualizations in Figure \ref{gmapattack} of the different transformed states is matched by their ranking under $\Lambda(S^\Psi,S)$, i.e. sorting from largest to smallest correlation quotient for BankHeist yields Nesterov momentum, Elastic-Net, and then Carlini\&Wagner. The fact that the feature correlation quotient has distinct results for untransformed states and for states under all the types of adversarial formulations indicates that RA-NLD can facilitate detecting different types of adversarial perturbations.

Measuring stimulus response to visual illusions has been used as an analysis tool in neural processing \citep{hubel62,alex96,gerald08,seymour18}.
One way to understand our approach is to examine the studies that focus on investigating the cortical area, parahippocampal cortex and hippocampus against visual illusion stimulus \citep{alex96,axel17}.
\begin{figure}[t]
\centering
\vskip -0.14in
\footnotesize
\stackunder[2pt]{\includegraphics[scale=0.16]{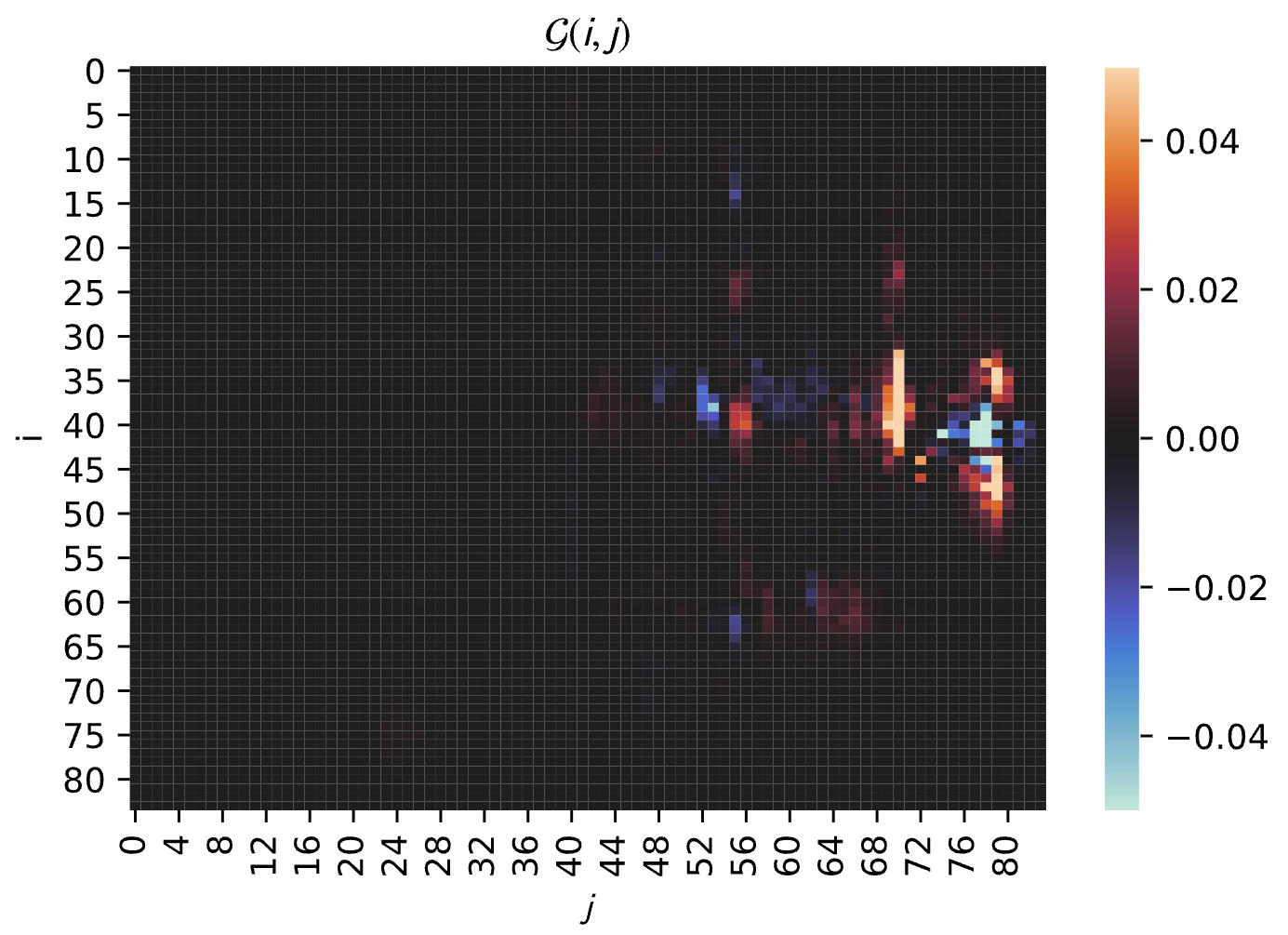}}{RoadRunner}
\hskip 0.1pt
\stackunder[2pt]{\includegraphics[scale=0.16]{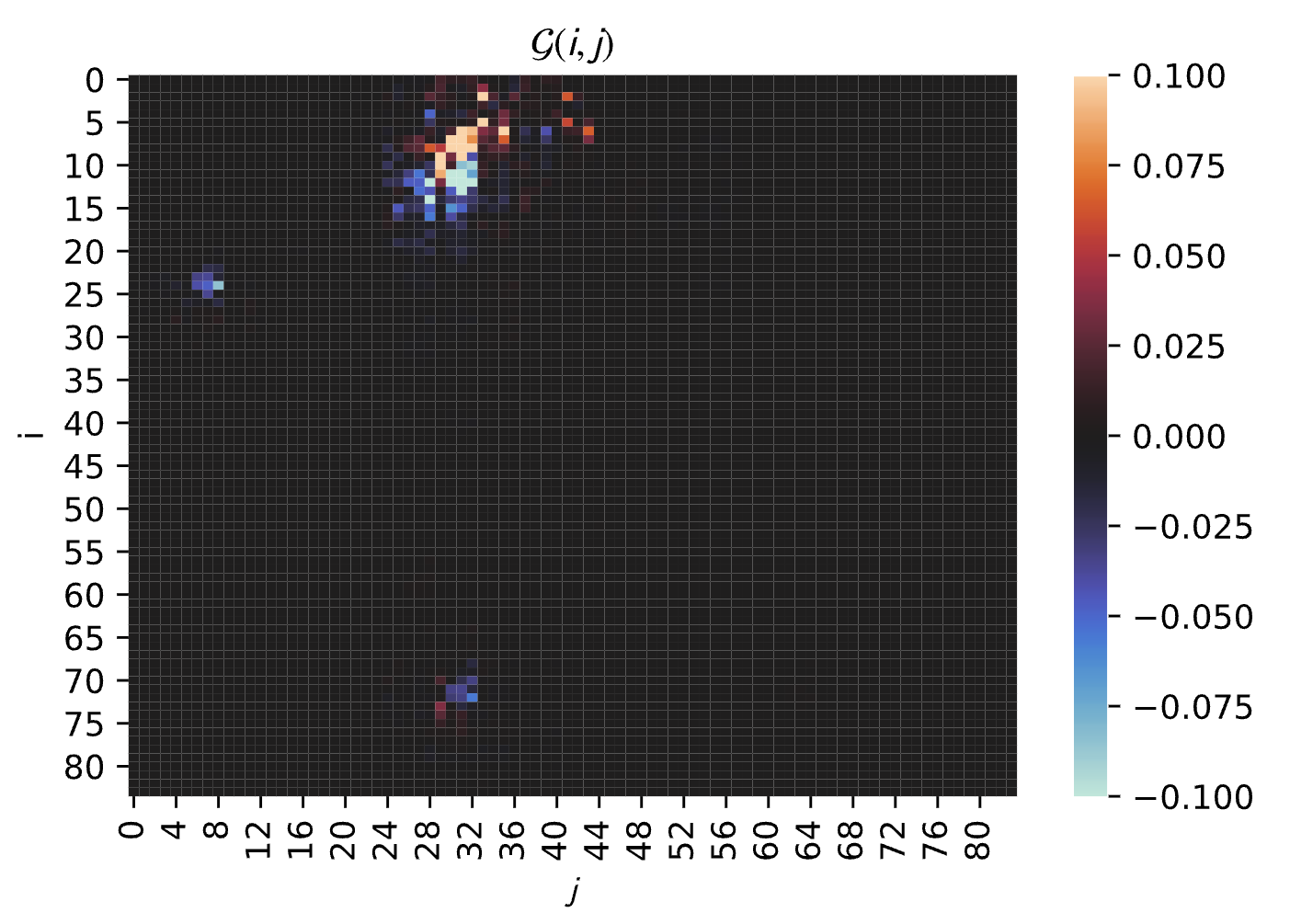}}{BankHeist}\\
\stackunder[2pt]{\includegraphics[scale=0.16]{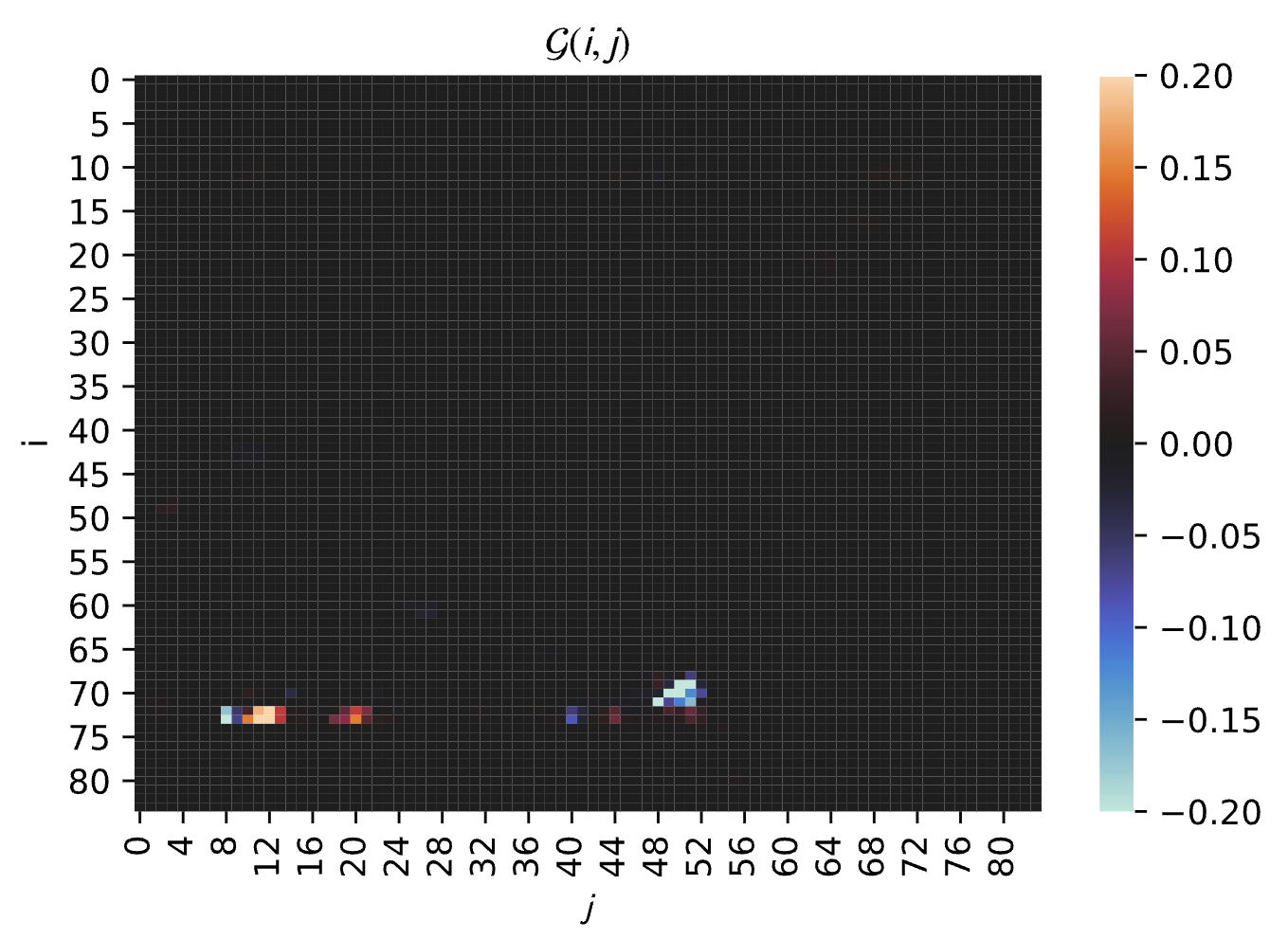}}{Pong}
\stackunder[2pt]{\includegraphics[scale=0.16]{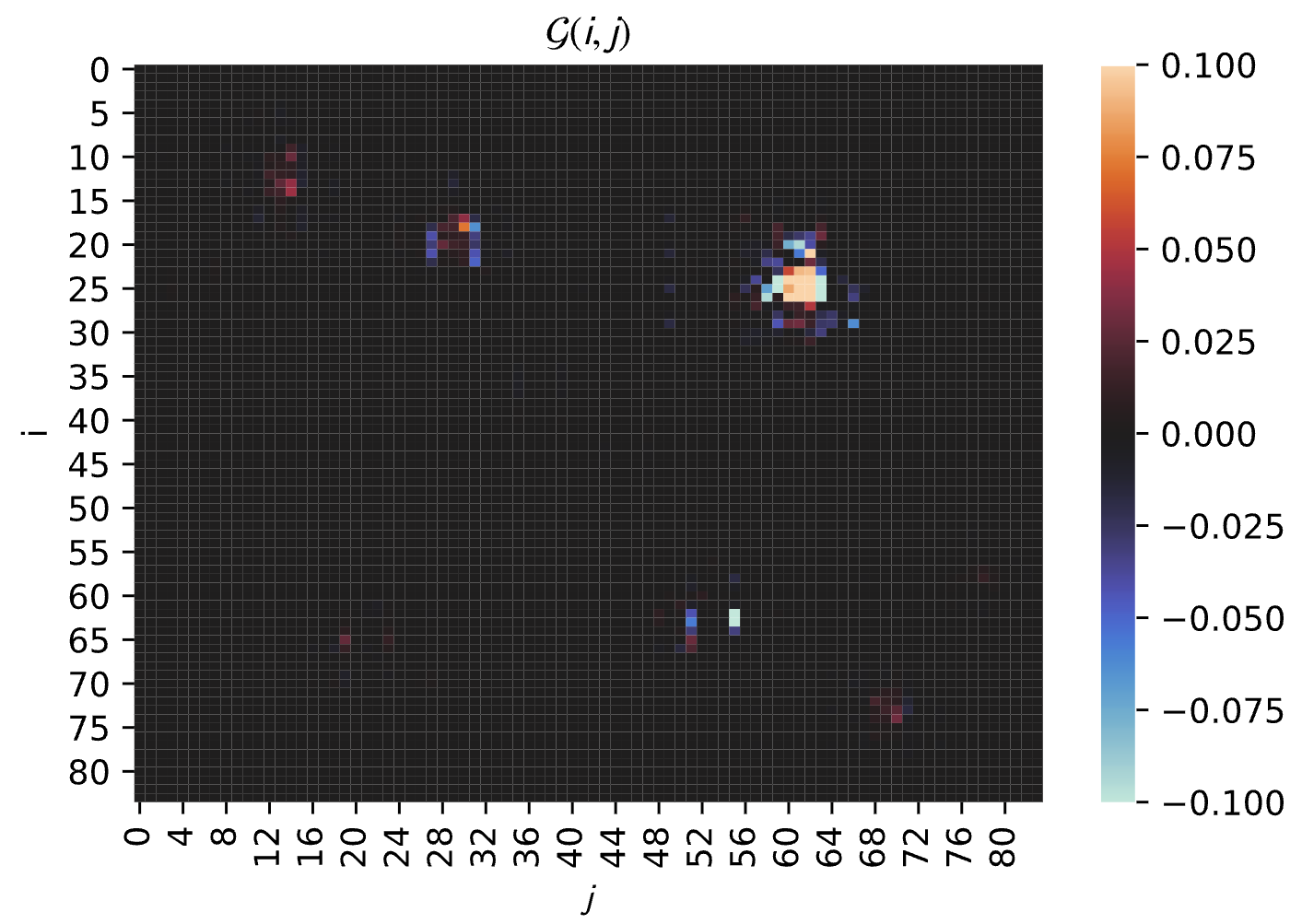}}{Freeway}
\vskip -0.12in
\caption{Principal non-Lipschitz direction $\mathcal{G}(i,j)$ for the state-of-the-art certified adversarially trained deep reinforcement learning policies for BankHeist, Pong, Freeway and RoadRunner.}
\label{advgmap}
\vskip -0.12in
\end{figure}

\begin{table*}[t]
\caption{The feature correlation quotient $\Lambda(S',S)$ in BankHeist, Freeway, RoadRunner, and Pong for the natural transformations: brightness and contrast, compression artifacts, rotation modification, perspective transform, blurred observations.}
\label{ent}
\centering
\scalebox{0.94}{
\begin{tabular}{lccccccr}
\toprule
Distributional Shift     &  Freeway
                				  &  BankHeist
                				  &  RoadRunner
                          &  Pong \\
\midrule
Untransformed States       & 0.9917$\pm$0.0023      & 0.8360$\pm$0.0116   & 0.7652$\pm$0.0385    & 0.4934$\pm$0.0391     \\
Brightness and Contrast    & 0.86756$\pm$0.0271     & 0.3095$\pm$0.0429   & 0.4369$\pm$0.0334    & 0.1678$\pm$0.0427     \\
Compression Artifacts      & 0.90564$\pm$0.237      & 0.38814$\pm$0.022   & 0.24358$\pm$0.0204   & 0.49341$\pm$0.0191    \\
Rotation Modification	   	 & 0.1381$\pm$0.0081      & 0.2951$\pm$0.0062   & 0.3350$\pm$0.0050    & 0.13648$\pm$0.0032    \\
Perspective Transform      & 0.3010$\pm$0.0281      & 0.1723$\pm$0.0311   & 0.3308$\pm$0.0274    & 0.4278$\pm$0.0196    \\
Blurred Observations       & 0.2657$\pm$0.0148      & 0.0954$\pm$0.0127   & 0.2496$\pm$0.0162    & 0.0847$\pm$0.0083     \\
 \bottomrule
\end{tabular}
}
\end{table*}
\footnotetext{Figure \ref{advft} reports the Fourier transform of $\mathcal{G}_S$ where $S$ is collected from a vanilla and adversarially trained policies in RoadRunner, BankHeist, Pong and Freeway. The Fourier transform reveals clear differences in the spatial frequencies occupied by $\mathcal{G}_S$ under vanilla and adversarial training. There is a consistent trend that the larger entries of the Fourier transform are more evenly and smoothly spread out for the adversarially trained policies. Thus, adversarial training leaves a consistent signature on the non-robust features detectable via the Fourier transform of $\mathcal{G}_S$. There is also a change in orientation: if the larger entries of the Fourier transform for the vanilla trained policy are more spread out along one axis, the adversarially trained Fourier transform is more spread along the other.}

\subsection{Vulnerable Representations Learnt via Certified Adversarial Training}
\label{advtrain}
In this section we investigate the effects of adversarial training on the correlated non-robust features. 
In particular, the SA-DDQN algorithm adds the regularizer $\mathcal{R}$, 
\[
  \mathcal{R}(\theta) = \sum_s \left(\max_{\bar{s} \in D_{\epsilon}(s)} \max_{a\neq a^*(s) } Q_{\theta}(\bar{s},a) - Q_{\theta}(\bar{s},a^*(s)) \right).
\]
during training in the temporal difference loss.
Figure \ref{advgmap} shows the RA-NLD results for the state-of-the-art adversarially trained deep reinforcement learning policies.
The non-robust features of the adversarially trained deep neural policies are much more tightly concentrated on disjoint coordinates in the state observations,
and these areas of concentration have moved significantly from where they were under vanilla training.
Thus, the visualization allows us to see that correlated, non-robust features persist in adversarially trained policies, albeit in different locations with disjoint patterns than vanilla trained deep reinforcement learning policies.
To complete our analysis of adversarial training we further include results on how non-robust features vary across time. For this purpose the $\ell_2$-norm of the gradient $\lVert \nabla_{s_g} J(s_i,s_g) \rVert^2$ in each state $s_i \in S$ is recorded
for both adversarially trained and vanilla trained policies in RoadRunner, Pong, and Freeway. The results are plotted in Figure \ref{advtrace}.
In both RoadRunner and Freeway, the adversarially trained policy has much higher variance in the gradient norm and thus in the level of instability.
This is in contrast to the vanilla trained policy which tends to have a much smoother distribution which remains closer to the mean. These results indicate that adversarial training introduces higher jumps in sensitivity over states (i.e. extreme instability) when compared to vanilla training.

\begin{figure*}[t]
	\vskip -0.08in
\centering
\footnotesize
\stackunder[0.2pt]{\includegraphics[scale=0.214]{ponggmapnorm2.png}}{Untransformed}
\hskip 0.1pt
\stackunder[0.2pt]{\includegraphics[scale=0.208]{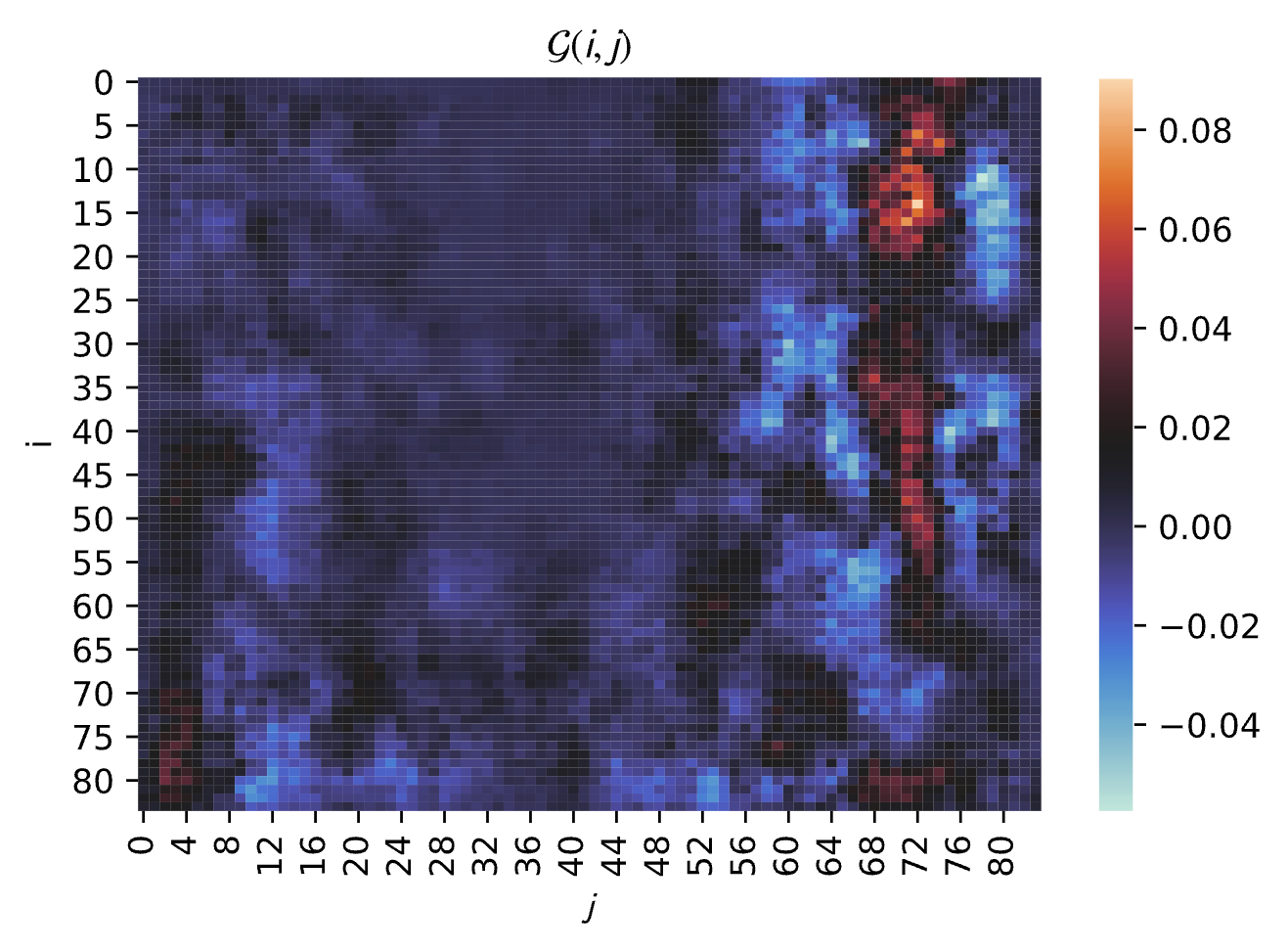}}{Rotated}
\stackunder[0.2pt]{\includegraphics[scale=0.20]{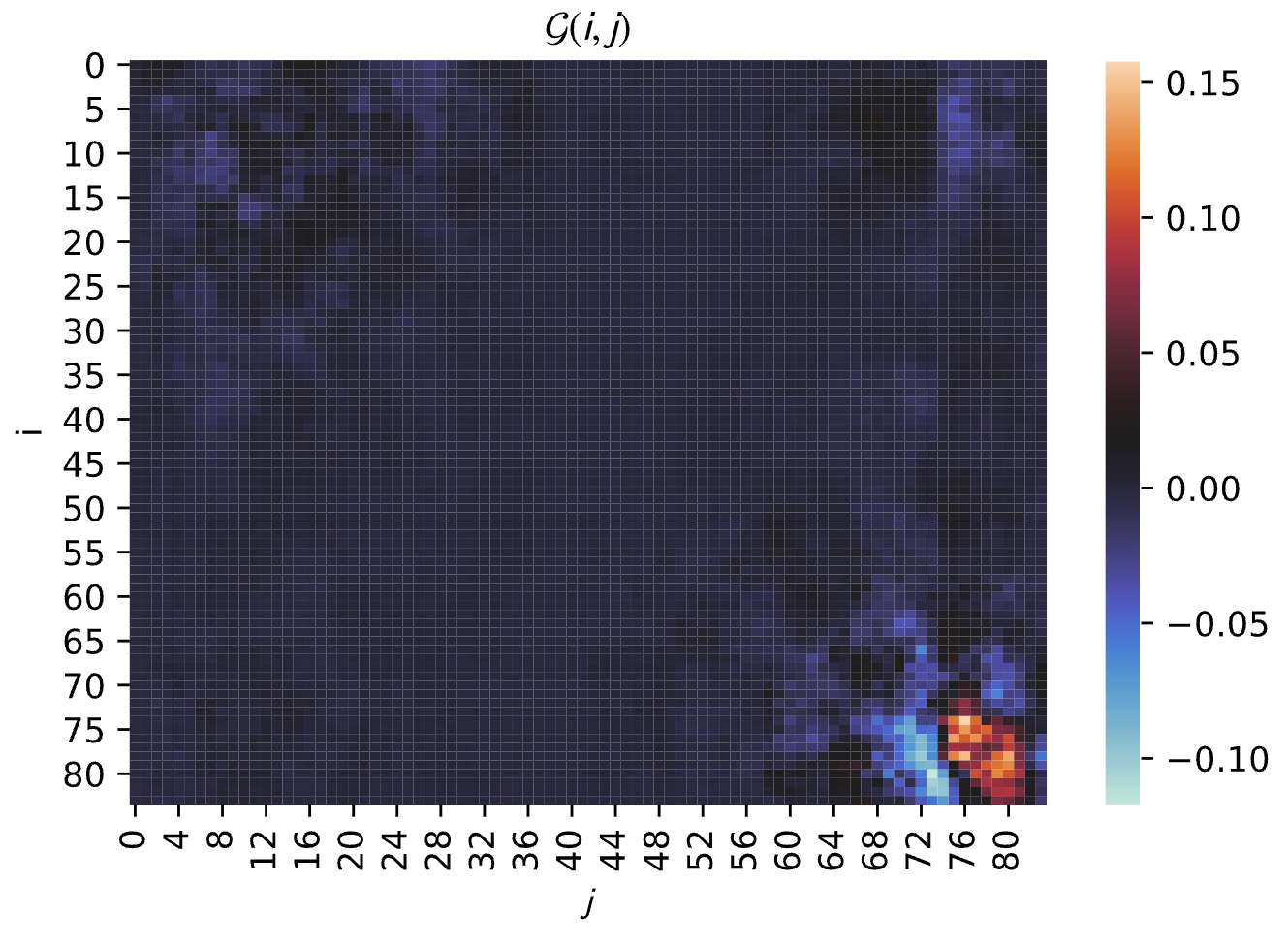}}{Blurred}\\
\vskip -0.01in
\stackunder[0.2pt]{\includegraphics[scale=0.20]{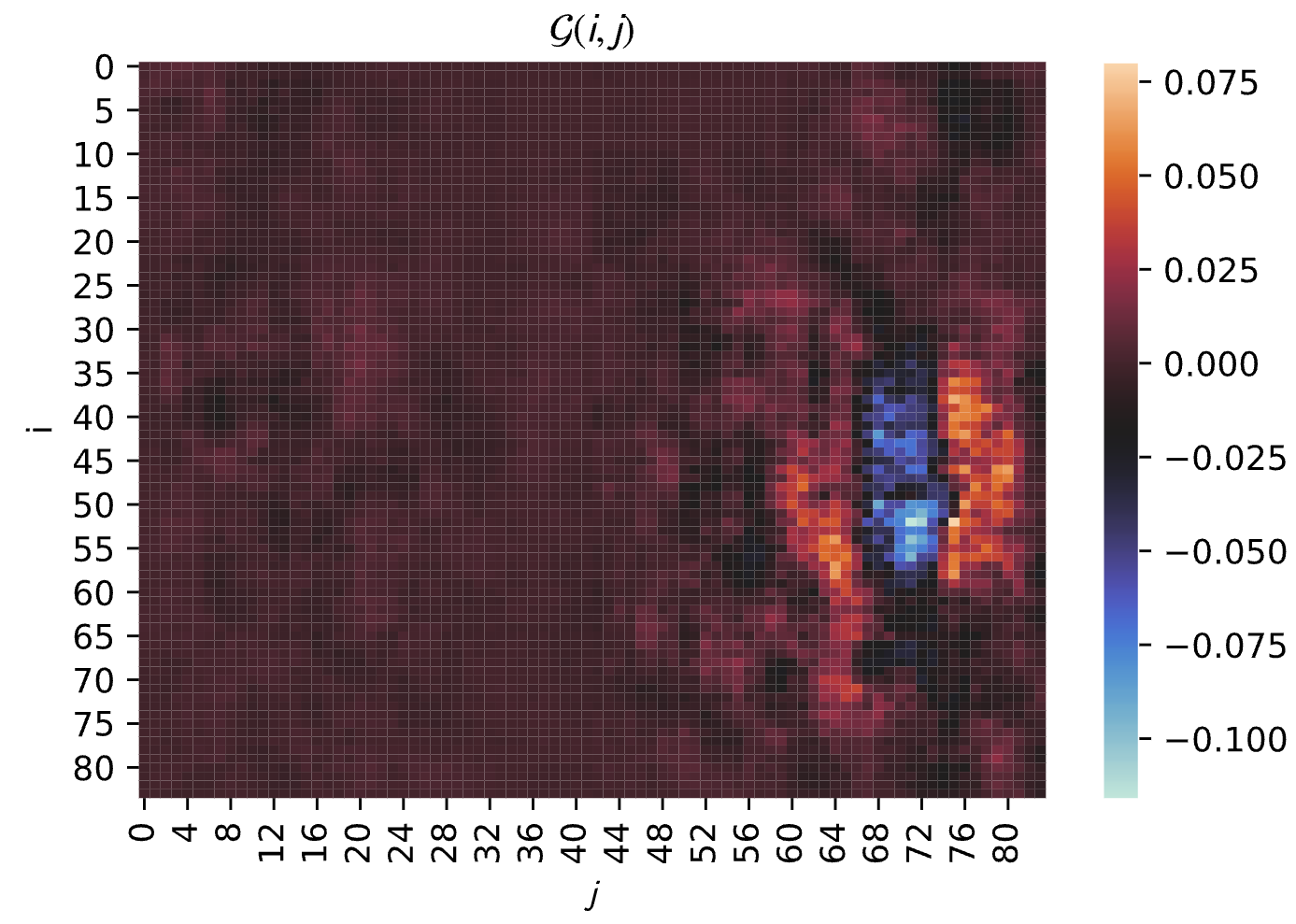}}{Compression Artifacts}
\stackunder[1pt]{\includegraphics[scale=0.20]{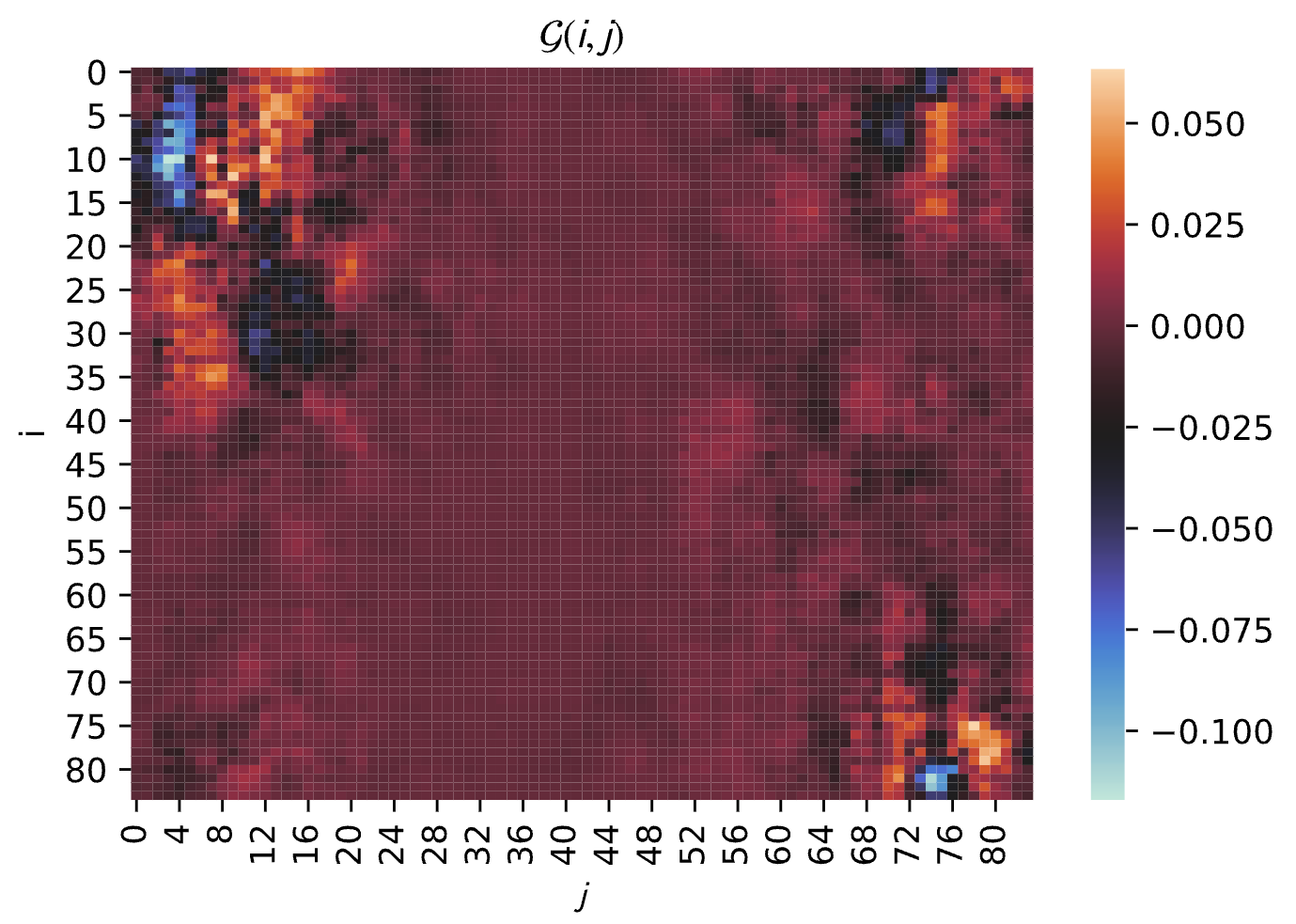}}{Perspective Transform}
\stackunder[1pt]{\includegraphics[scale=0.20]{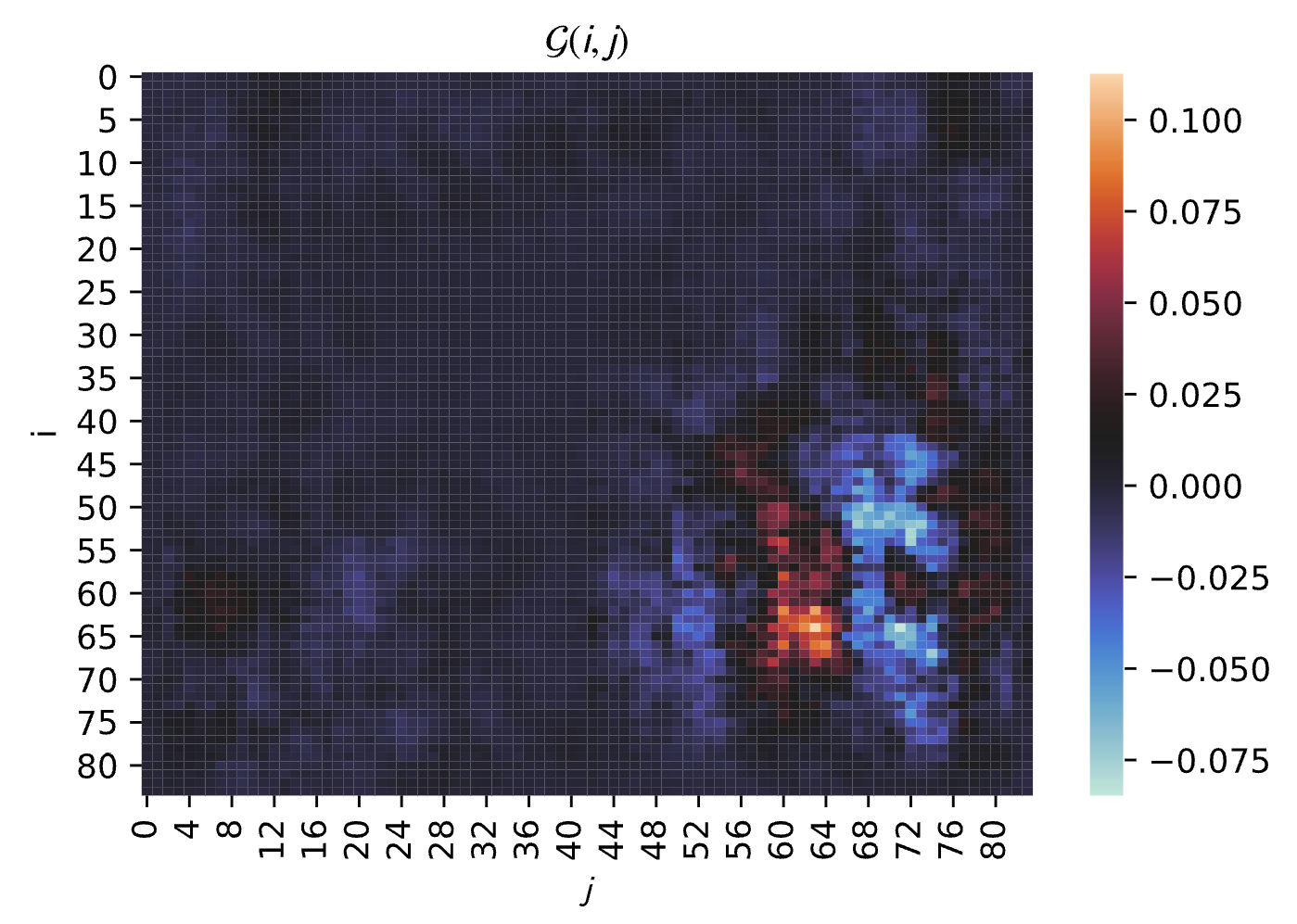}}{Brightness and Contrast}
\vskip -0.12in
\caption{RA-NLD results of untransformed state observations and states under natural transformations with rotation, perspective transformation, blurring, compression artifacts, and B\&C for Pong.}
\label{pongenv}
% \vskip -0.1in
\end{figure*}

\subsection{The Effects of Imperceptible Distributional Shift on the Directions of Instabilities}

To evaluate the effects of distributional shift on the learnt policy we provide analysis on several environment modifications with RA-NLD. These transformations are natural semantically meaningful changes to the given MDP that correspond to imperceptible modifications to the state observations. In particular, the imperceptibility $\mathcal{P}_{\textrm{similarity}}$ is measured by, $\mathcal{P}_{\textrm{similarity}}(s,\Psi(s)) = \sum_l\frac{1}{H_l W_l}\sum_{h,w} \lVert w_l \odot (\hat{y}^l_{shw} - \hat{y}^l_{\Psi(s)hw}) \rVert_2^2 $
where $\hat{y}_s^l,\hat{y}_{\Psi(s)}^l \in \mathbb{R}^{W_l\times H_l\times C_l}$ represent the vector of unit normalized activations in the convolutional layers with width $W_l$, height $H_l$, and $C_l$ is the number of channels.\footnote{These imperceptible transformations include perspective transform, blurring, rotation, brightness, contrast, and compression artifacts as proposed in \citet{korkmaz2023aaai}.
In particular, brightness and contrast is given by linear transformation, and compression artifacts are the diminution in high frequency components due to JPEG conversion.
Note that this recent work demonstrates that these natural imperceptible transformations cause more damage to the policy performance compared to adversarial perturbations, and further highlights that the certified adversarial training is more vulnerable towards these natural attacks.
}
Figure \ref{pongenv} reports $\mathcal{G}_S$ for states $S$ collected under the six environment modifications mentioned above.
For the untransformed setting the visualization of $\mathcal{G}_S$ clearly emphasizes the center of the region where the agent's paddle moves up and down to hit the ball. The components of $\mathcal{G}_S$ take larger positive values at the center of this region and transition to negative values along the boundary. A similar emphasis can be found for the case of compression artifacts, but with the signs reversed (i.e. the center of the region is negative and the boundary is positive).
The other transformations exhibit larger changes in the regions emphasized in the visualization with perspective transform, blurring, rotation, and B\&C causing the emphasized region to move to different locations.
Table \ref{ent} contains the values of $\Lambda(\hat{S},S)$ and $\Lambda(S^\Psi,S)$ where $S$ is collected from an untransformed run and $S^\Psi$ is collected from each of the six different transformations. %The rows correspond to different games.
In every game the largest value of $\Lambda(\hat{S},S)$ occurs when $\hat{S}$ comes from an independent untransformed run, indicating that the additional decrease observed for $S^\Psi$ from transformed runs is caused by the respective environmental transformations. It is notable that in Pong the second highest value for $\Lambda(S^\Psi,S)$ occurs for $S^\Psi$ collected with compression artifacts, as this corresponds precisely to the qualitative similarity between the regions emphasized in the visualization of $\mathcal{G}_S$ for untransformed and compression artifacts.
Hence, the results for $\Lambda(S^\Psi,S)$ help us to quantitatively understand the effects of the environmental changes in the MDP, while agreeing well with the qualitative results of the RA-NLD outputs.

\subsection{RA-NLD to Understand Policy Decision Making and Diagnose Non-Robustness}
By leveraging the non-Lipschitz direction analysis not only can we uncover non-robust representations learnt deep neural policies, further we can analyze how their decisions are formed given an MDP and a training algorithm and what makes these decisions change under different influences from adversarial manipulations and natural changes in a given environment. 
While the RA-NLD visualizations give us semantically meaningful information on how policy decisions are influenced and the non-robust features learnt by the deep neural policy, they also provide a detailed understanding of how these volatile representations change under non-stationary MDPs.
The fact that RA-NLD can provide fine-grained vulnerability analysis of deep reinforcement learning policies under adversarial attacks, with distributional shift and with different training algorithms can help with diagnosis of policy vulnerabilities in the development phase.
Conducting ablation studies with RA-NLD in reinforcement learning algorithm design can prevent building policies with inherent non-robustness, and our algorithm can be utilized to visualize and identify the effects of several design choices (e.g. algorithm, neural network architecture) on the volatile patterns learnt by the policy from the MDP.
In particular, given a visualization of the vulnerability pattern for a trained policy, one can try to modify the training environment in a way that will make the policy invariant to the non-robust features revealed by RA-NLD. Such modification could include changing the state representation in a way that does not change the semantics of the MDP or the task at hand, but does change the inherent non-robustness in question. Furthermore, the effect of modifications to training algorithms can  also be directly visualized, as exemplified by our results for adversarial training. Thus our method gives a straightforward way to diagnose or debug any proposed methods in terms of their effects on
the non-robustness of the neural policy and the volatile representations learnt by it.

One intriguing fact is that RA-NLD can uncover the vulnerable representations learnt by the certified adversarial training techniques. From the safety point of view it warrants significant concern that the algorithms targeting and certifying robustness end up learning non-robust representations. From the alignment perspective RA-NLD discovers that certified adversarial training is still producing misaligned deep reinforcement learning policies. Ultimately, for future research directions it is important to lay out exact trade-offs and vulnerabilities for these algorithms to eliminate the bias they can create for future research efforts.
The impact of the imperceptible environmental changes in the MDP is immediately captured by the principal high-sensitivity direction analysis.
The most intriguing aspect of these results is that not only can RA-NLD be used as a diagnostic tool during training, but further the principal non-Lipschitz direction analysis can also guide agents in real life on real-time understanding of the current rationale behind their decisions and their vulnerabilities.
 The RA-NLD algorithm gives us semantically meaningful information on the non-robust features learnt by the deep neural policy, and also provides a detailed understanding of how these non-robust features change under non-stationary environments.

\section{Conclusion}

In our paper we aim to seek answers for the following questions: \textit{
(i) How can we analyze the robustness and reliability of deep reinforcement learning policy decisions?
(ii) What is the relation of non-robust representations learnt by deep neural policy temporally and spatially?
(iii) How do adversarial attacks affect the correlated volatile representations learnt by deep reinforcement learning policies?
(iv) Does adversarial training ensure safety and provide robust policies that do not learn non-robust representations? 
(v) How does distributional shift affect the learnt correlated non-robust features?} To be able to answer these questions we analyze non-Lipschitz directions in the deep neural policy manifold and we propose a novel technique to analyze and lay out correlated non-robust representations learned by deep reinforcement learning policies. We show that deep reinforcement learning policies do end up learning correlated non-robust vulnerable representations, and that adversarial attacks lead to surfacing a new set of non-robust features or highlighting the existing ones. Most importantly, our results show that the state-of-the-art adversarial training techniques, i.e. robust deep reinforcement learning, also end up learning temporally and spatially correlated non-robust features. 
Finally, we demonstrate that distributional shifts introduce different sets of correlated non-robust features compared to adversarial attacks. Hence, our analysis not only allows us to effectively visualize correlated directions of instability, but also allows for precise understanding of changes in the learnt non-robust representations caused by different training algorithms and different methods for altering states. Thus, we believe that our analysis can be critical both in understanding deep reinforcement learning policy decision making and in diagnosing the vulnerabilities of deep neural policies, while further enhancing our ability to design algorithms to improve robustness.

\section*{Impact Statement}

The risks of artificial intelligence regarding safety have never been as prominent as they are in the current time \citep{tobin23}.
From highly capable large language models \citep{gemini23,openai23} to autonomous driving vehicles, these risks arise in real life \citep{nyt23} as regulatory acts are being formed \citep{usa23,euro23,eu23}.
Our paper provides the necessary diagnostic tools to understand and interpret AI systems (i.e. deep reinforcement learning policies).
Our paper introduces a theoretically founded technique to understand the vulnerabilities and volatilities of deep neural policies.
Our results discover that \emph{certified robust} training techniques have spikier volatilities resulting in revealing the current problems of safety guarantees in adversarial training techniques.
We believe that it is crucial to understand the exact problems that might arise from the deep reinforcement learning policies before these policies are deployed in real life \citep{nyt22}.

\bibliography{example_paper2}%%%%%%%%%%%%%%%%%%%%%%%%%%%%%%%%%%%%%%%%%%%%%%%%%
\bibliographystyle{icml2024}

\end{document}